%% file: main.tex
\input{_0_vars}

\ifdefined\isarxiv
\documentclass[11pt]{article}
\usepackage[numbers]{natbib}
\else
\documentclass{article}
\usepackage{neurips_2025}
\fi

\ifdefined\isarxiv

\usepackage{amsmath}
\usepackage{amsthm}
\usepackage{amssymb}
\usepackage{algorithm}
\usepackage{subfig}
\usepackage{algpseudocode}
\usepackage{graphicx}
\usepackage{grffile}
\usepackage{wrapfig,epsfig}
\usepackage{url}
\usepackage{xcolor}
\usepackage{epstopdf}

\usepackage{bbm}
\usepackage{dsfont}

\else 

\usepackage[utf8]{inputenc} 
\usepackage[T1]{fontenc}    
\usepackage{hyperref}       
\usepackage{url}            
\usepackage{booktabs}       
\usepackage{amsfonts}       
\usepackage{nicefrac}       
\usepackage{microtype}      
\usepackage{xcolor}         

\fi

\ifdefined\isarxiv

\else
\usepackage{amsmath}
\usepackage{amsthm}
\usepackage{amssymb}
\usepackage{algorithm}
\usepackage{subfig}
\usepackage{algpseudocode}
\usepackage{graphicx}
\usepackage{grffile}
\usepackage{wrapfig,epsfig}
\usepackage{epstopdf}
\usepackage{bbm}
\usepackage{dsfont}
\fi
 
\allowdisplaybreaks

\ifdefined\isarxiv

\usepackage{tikz}
\usepackage{hyperref}  
\hypersetup{colorlinks=true,citecolor=blue,linkcolor=blue} 
\usetikzlibrary{arrows}
\usepackage[margin=1in]{geometry}

\else
\definecolor{mydarkblue}{rgb}{0,0.08,0.45}
\hypersetup{colorlinks=true, citecolor=mydarkblue,linkcolor=mydarkblue}
\fi
 
\graphicspath{{./figs/}}

\theoremstyle{plain}
\newtheorem{theorem}{Theorem}[section]
\newtheorem{lemma}[theorem]{Lemma}
\newtheorem{definition}[theorem]{Definition}

\newtheorem{assumption}[theorem]{Assumption}

\input{_1_maths}

\begin{document}

\ifdefined\isarxiv

\date{}
\title{\paperTitle}
\author{\paperAuthor}

\else

\title{\paperTitle}

\author{%
  David S.~Hippocampus\thanks{Use footnote for providing further information
    about author (webpage, alternative address)---\emph{not} for acknowledging
    funding agencies.} \\
  Department of Computer Science\\
  Cranberry-Lemon University\\
  Pittsburgh, PA 15213 \\
  \texttt{hippo@cs.cranberry-lemon.edu} \\
}

\maketitle

\fi

\ifdefined\isarxiv
\begin{titlepage}
  \maketitle
  \begin{abstract}
    \input{00_abstract}

  \end{abstract}
  \thispagestyle{empty}
\end{titlepage}

{\hypersetup{linkcolor=black}
\tableofcontents
}
\newpage

\else

\begin{abstract}
\input{00_abstract}
\end{abstract}

\fi


\input{_2_body}

\ifdefined\isarxiv
\bibliographystyle{alpha}
\bibliography{ref}
\else
\bibliographystyle{alpha}
\bibliography{ref}
\fi


\newpage
\onecolumn
\appendix

\begin{center}
    \textbf{\LARGE Appendix }
\end{center}

\input{_3_app}


\ifdefined\isarxiv
\else
\input{checklist}

\fi

\end{document}

%% file: _0_vars.tex
\def\isarxiv{1}

\def\paperTitle{SORSA: Singular Values and Orthonormal Regularized Singular Vectors Adaptation of Large Language Models}

\def\paperAuthor{
Yang Cao\thanks{\texttt{ycao4@wyomingseminary.org}. Wyoming Seminary.}
\and
Zhao Song\thanks{\texttt{magic.linuxkde@gmail.com}. UC Berkeley.}
}



%% file: _1_maths.tex
\newcommand{\R}{\mathbb{R}}

\newcommand{\train}{\mathrm{train}}
\newcommand{\reg}{\mathrm{reg}}
\newcommand{\unreg}{\mathrm{unreg}}

\DeclareMathOperator{\diag}{diag}

\newcommand{\method}{\mathrm{SORSA}}

%% file: 00_abstract.tex
In this paper, we propose Singular Values and Orthonormal Regularized Singular Vectors Adaptation, or SORSA, a novel parameter efficient fine-tuning (PEFT) method. Each SORSA adapter consists of two main parts: trainable principal singular weights $W_p = U_p \text{diag}(S_p) V^\top_p$, and frozen residual weights $W_r = U_r \text{diag}(S_r) V^\top_r$. These parts are initialized by performing singular value decomposition (SVD) on pre-trained weights. Moreover, we implement and analyze an orthonormal regularizer, which we prove could decrease the condition number of $W_p$ and make the optimization more efficient. SORSA adapters could be merged during inference, thus eliminating any inference latency. 
We also introduce a method to analyze the variation of the parameters by performing SVD and discuss and analyze SORSA's superiority in minimizing the alteration in the SVD aspect.
After all, SORSA shows a faster convergence than LoRA and PiSSA in our experiments. On the GSM-8K benchmark, Llama 2 7B adapted using SORSA achieved 56.03\% accuracy, surpassing LoRA (42.30\%) and Full FT (49.05\%).
We conclude that SORSA offers a new perspective on parameter-efficient fine-tuning, demonstrating remarkable performance.

%% file: _2_body.tex
\input{01_intro}
\input{02_related_work}
\input{03_sorsa}
\input{05_theory}
\input{06_exp}
\input{07_conclusion}

%% file: 01_intro.tex
\section{Introduction}
\label{sec:introduction}

Pre-trained large language models (LLMs) demonstrate strong generalization capabilities, enabling them to perform a wide range of natural language processing (NLP) tasks \cite{bmr+20,aaa+23,tli+23,pga+24,gdj+24}. For adapting LLMs to specific downstream tasks, the default approach is often full parameter fine-tuning (Full FT), which updates all model parameters.

However, as LLMs continue to grow in scale, Full FT becomes increasingly impractical due to high computational and memory demands. To alleviate this, Parameter-Efficient Fine-Tuning (PEFT) methods have gained popularity, offering a cost-effective alternative by only updating a small subset of parameters.

Among PEFT approaches, LoRA~\cite{hsw+22} has emerged as a preferred choice due to its simplicity, efficiency, and minimal impact on inference-time latency. LoRA injects low-rank trainable matrices into the model, enabling effective fine-tuning with significantly reduced resource requirements.

Despite its efficiency, LoRA and similar PEFT methods face a major challenge under low-data regimes: they tend to overfit and degrade the model’s original generalization ability, and even cause catastrophic forgetting \cite{xlz+21,lmc+24,sats24,vsk24}. For instance, fine-tuning on a small mathematical dataset may cause the model to forget previously acquired capabilities such as code generation or commonsense reasoning.

Previous works \cite{ssk18,swl24,flgw25} have shown that neural networks with well-conditioned weight is able to provide a more robust performance. We further analyze this phenomenon in the context of PEFT, and identify the condition number of weight matrices as a critical factor affecting generalization during fine-tuning.
Our study shows that LoRA often amplifies the condition number, making the adapted model increasingly ill-conditioned and unstable.

To address this, we propose a new PEFT method that explicitly improves the conditioning of the model during training. Our approach introduces orthonormal regularization to maintain well-conditioned weights, thereby preserving the model’s generalization while enabling efficient adaptation. Empirical results show that our method significantly mitigates overfitting and outperforms existing baselines across various tasks.

We summarize our main contributions as follows:
\begin{itemize}
    \item We demonstrate that during PEFT, well-conditioned weights tend to have better generalization.
    \item We propose $\method$, a novel parameter-efficient fine-tuning (PEFT) method that combines low-rank SVD-based initialization with orthonormal regularization.
    \item We provide the convergence rate of $\method$ with gradient descent. (Theorem~\ref{thm:rate})
    \item We provide theoretical analysis showing that the orthonormal regularizer leads to better-conditioned weight updates. (Theorem~\ref{thm:condition})
    \item We provide a novel approach to analyze the parameter alteration during fine-tuning.
    \item We empirically demonstrate that $\method$ consistently outperforms or matches the performance of strong baselines, including full fine-tuning, LoRA, PiSSA, and AdaLoRA.
\end{itemize}

%% file: 02_related_work.tex
\section{Related Work}
\label{sec:related_work}




\paragraph{Efficient Computation in Machine Learning.}

As the increasing scale of training data and model parameters, developing efficient machine learning algorithms have become central focus of recent AI research.
In visual recognition, the acceleration of CNN~\cite{on15,hzrs16} and ViT~\cite{dbk+20} have long been a heated topic, especially for edge devices that have limited computation resources. Representative acceleration techniques including architectural simplification~\cite{shz+18,dzm+21}, quantization~\cite{wlw+16,lwh+21}, and pruning~\cite{yhw+22}. These techniques have significantly advance in real world applications, e.g. autonomous driving~\cite{jcx+23}, medical image segmentation~\cite{hjw22}, remote sensing~\cite{xzz+21},
emotion recognition~\cite{zlz21_micro,zlz21_face,lws+22}, and industrial automation. 
In content creation, diffusion models~\cite{hja20,rbl+22} and flow matching models~\cite{lcb+22,lgl22} are high-fidelity visual content generators. Acceleration in this area focuses on model architecture design~\cite{dpnt23,fhla24,cgl+25,ccl+25}, fast ODE sampler~\cite{xlc+24}, complexity analysis~\cite{gpp+24,kll+25}, distillation~\cite{mrg+23}. These works have inspired many future applications, e.g. education, drug discovery~\cite{wxhl24}, face synthesis~\cite{lwz+22}, and advertising~\cite{lzw+24}, and directions, e.g. benchmarks~\cite{cgh+25,ghh+25,ghs+25_physical,ghs+25_text} and theoretical explorations~\cite{hwl+24}. Graph Neural Networks are fundamental tools to model complex relational data~\cite{vcs+18,xhlj19,lls+25}, where important acceleration techniques include sparsification~\cite{mrm20,lcz+23}, GNN to MLP distillation~\cite{zlss22,hzl+23}, and lazy computation~\cite{nsj+21,zxf+24,xht+24}. These techniques has inspired applications including but not limited to spatio-temporal data mining~\cite{zwz+22,wac+22}, fake news detection~\cite{xwl+22,chl+24_gat}, human skeleton-based visual recognition~\cite{lcc+21,flz+21}, while also inspired aspects of graph neural networks including mitigating sensitive data influence~\cite{cpm23,zha24,yw25}, and robustness~\cite{gss+21,dlf+21}.

\paragraph{PEFT Methods.}

PEFT methods have been proposed to alleviate the inefficiency of full-parameter fine-tuning for large language models. These methods update only a small subset of parameters, often keeping the majority of the pre-trained model frozen, which significantly reduces memory and computational costs during training.

Adapter-based approaches were among the earliest PEFT methods, introduced by \cite{hgj+19}, where small trainable modules are inserted between frozen layers. Subsequent works such as \cite{lmf20} and \cite{hzm+21} explored more compact or parallelized adapter designs. However, all adapter-based methods generally incur additional inference-time latency, since the inserted modules are not mergeable with the original model weights.

LoRA~\cite{hsw+22} gained popularity for introducing low-rank trainable matrices added to the pre-trained weight matrices. This approach avoids inference latency while offering competitive performance. Variants of LoRA expand upon this idea: AdaLoRA~\cite{zcb+23} improves parameter efficiency by incorporating dynamic rank selection via singular value decomposition and pruning. DoRA~\cite{lwy+24} decouples the direction and magnitude of weight updates, achieving higher expressiveness at the cost of higher training-time computation. OLoRA~\cite{b24} uses orthogonal initialization via QR decomposition to improve convergence speed. PiSSA~\cite{mwz24} decomposes the pre-trained weight matrix and isolates a residual component, which remains frozen during training to improve convergence and stability.

Prompt-based PEFT methods, such as prefix-tuning~\cite{lac21}, prepend learnable tokens to the model input. Although these methods are simple to implement, they often lead to longer input sequences and require careful prompt engineering. Other recent advances include GaLore~\cite{zzc+24}, which reduces memory usage through low-rank gradient accumulation, and LISA~\cite{pld+24}, which selectively fine-tunes critical layers using layer-wise importance sampling.

\paragraph{Condition Numbers in Neural Networks}

\section{Preliminary} \label{sec:preli}

In this section, we first introduce our notations, then provide preliminary for our work.

\subsection{Notations}
We used $\R$ to denote real numbers. We use $A \in \R^{n \times d}$ to denote an $n \times d$ size matrix where each entry is a real number.
We use $I_d$ to denote the $d \times d$ identity matrix.
We use $A^\top$ to denote the transpose of a matrix $A$. 
We use $A^{1/2}$ to denote element-wise square root of the matrix $A$, i.e. $(A^{1/2})_{i,j} = (A_{i,j})^{1/2}$.
We use $\|A\|_F$ to denote Frobenius norm of matrix $A$.
We use $\|A\|$ to denote spectral norm of matrix $A$.
We use $A \preceq B$ to denote the positive semidefinite order, i.e. for symmetric $A, B \in \R^{d \times d}$, $
A \preceq B \Longleftrightarrow B - A \succeq 0$.

\subsection{PEFT Methods}

\paragraph{LoRA} LoRA~\cite{hsw+22} represents the weight as a low-rank decomposition:
\begin{align*}
    W = W_0 + BA,
\end{align*}
where $W_0 \in \R^{m \times n}$ is the frozen pre-trained weight, $A \in \R^{m \times r}$ is Gaussian-initialized, and $B \in \R^{r \times n}$ is initialized with zeros.

\paragraph{AdaLoRA.} AdaLoRA~\cite{zcb+23} introduces dynamic rank adaptation via SVD, and prunes less significant singular values to reduce parameter overhead.

\paragraph{DoRA.} DoRA~\cite{lwy+24} reformulates the weight update as a normalized decomposition:
\[
W = m \cdot \frac{W_0 + BA}{\|W_0 + BA\|_c},
\]
where $m = \|W_0 + BA\|_c$ is the column-wise norm. This improves model capacity but increases computational cost per step.

\paragraph{OLoRA.} OLoRA~\cite{b24} initializes $A$ and $B$ using QR decomposition, ensuring orthonormality in the initial adapter weights, which empirically speeds up convergence.

\paragraph{PiSSA.} PiSSA~\cite{mwz24} decomposes $W_0$ via SVD as $W_0 = U \Sigma V^\top$ and splits it into:
\begin{align*}
W_\mathrm{pri} = AB,\quad \text{where } A = U_p S_p^{1/2},\quad B = S^{1/2}_p V^\top_p,
\end{align*}
with $U_p, S_p, V_p$ being the top-$r$ components. The residual $W_\mathrm{res} = U_r S_r V_r^\top$ remains frozen during training. This results in faster convergence and improved model fit.

\subsection{Condition Number}
We here provide a formal definition for the condition number.

\begin{definition}[Condition Number] \label{def:kappa}
Let $A \in \R^{m \times n}$ be a matrix with full column rank. The \emph{condition number} of $A$ with respect to the spectral norm is defined as
\begin{align*}
    \kappa(A) := \frac{\sigma_{\max}(A)}{\sigma_{\min}(A)} = \|A\| \cdot \|A^{-1}\|,
\end{align*}
where $\sigma_{\max}(A)$ and $\sigma_{\min}(A)$ are the largest and smallest nonzero singular values of $A$.
\end{definition}

%% file: 03_sorsa.tex
\section{Our Method}
\label{sec:method}

Giving a matrix $W \in \R^{m \times n}$, with $m \geq n$ (without loss of generality), we could perform SVD to decompose $W$ by $W = U\diag(S) V^\top$. Here, $U \in \R^{m \times k}$ is a matrix of left singular vectors and has orthonormal columns, $V \in \R^{n \times k}$ is a matrix of right singular vectors and has orthonormal columns, and $S \in \R^{k}$ are singular values $\sigma^1,\sigma^2 \ldots \sigma^k$ arranged in descending order. $\diag(S)$ is constructed by placing the elements of $S \in \R^k$ along the main diagonal, with all other elements zero.

According to our SVD notations, given a rank $r$ where $r \ll k$, we could perform the low-rank approximation by selecting the first $r$ items on the diagonal of $\Sigma$, which is the first $r$ most significant singular values, and also select the first $r$ columns of $U$ and first $r$ rows of $V^\top$, which correspond to the selected singular values. By performing SVD low-rank approximation, we could get a low-rank matrix that preserves the largest significant values and vectors, containing the matrix's ``most essential'' data.

We use $\Sigma_p \in \R^{n \times n}$ to denote a diagonal matrix where first $r$ entries are non-zero and all the remaining $n-r$ entries. Similarly, we use $\Sigma_r \in \R^{n \times n}$ to denote a diagonal matrix where first $n-r$ entries are non-zero and all the remaining $r$ entries are zeros. Let $\Sigma = \Sigma_p + \Sigma_r$. Let SVD of $W$ be $W = U \Sigma V^\top$.

Therefore, for a pre-trained weight $W_0 \in \R^{m \times n}$, we could split it based on its singular value into principal weight $W_p$ and residual weight $W_r$,
\begin{align*}
W_p := \underbrace{ U }_{m \times n} \underbrace{ \Sigma_p }_{n \times n} \underbrace{ V^\top }_{n \times n}  \in \R^{m \times n}, ~~~~ W_r := \underbrace{ U }_{m \times n} \underbrace{ \Sigma_r }_{n \times n} \underbrace{ V^\top }_{n \times n} \in \R^{m \times n}.
\end{align*}
Here, $U$ represents the matrix of left singular vectors, $S$ represents the singular values, $\diag(W)$ denotes a function to form a diagonal matrix from $W$, and $V$ represents the matrix of right singular vectors.
Since $\Sigma_p$ is zeroed out in the last $n-r$ entries, and $\Sigma_r$ is zeroed out in the first $r$ entries, we can easily find low-rank equivalents of $W_p$ and $W_r$. Specifically,
\begin{align*}
    W_p = \underbrace{U_p}_{m \times r} \underbrace{S_p}_{r \times r} \underbrace{V_p^\top}_{r \times n}, ~~~~~~ W_r = \underbrace{U_r}_{m \times (n-r)} \underbrace{S_r}_{(n-r) \times (n-r)} \underbrace{V_r^\top}_{(n - r) \times n},
\end{align*}
where $U_p$ is the first $r$ columns of $U$, $S_p$ is the first $r$ columns and rows of $\Sigma_p$, $V_p$ is the first $r$ columns of $V$, $U_r$ is the last $n-r$ columns of $U$, $S_r$ is the last $n-r$ columns and rows of $\Sigma_r$, $V_r$ is the last $n-r$ column of $V$.

The initialization of $W_r$ in $\method$ is same as PiSSA \cite{mwz24}. Nevertheless, unlike PiSSA which merge $S_p$ with $U_p$ and $V^\top_p$ into $A$ and $B$ by $A = U_p S_p^{1/2}$ and $B = S_p^{1/2}V^\top_p$, $\method$ remains $U_p$, $S_p$, and $V^\top_p$ in separate weight. $\method$ is defined by Eq.~\eqref{eq:def}, initially equivalent to the pre-trained weight $W_0$. During training, $W_r$ remains frozen, and only $U_p$, $S_p$, and $V^\top_p$ are updated.

$\method$ is defined as:
\begin{align} \label{eq:def}
\method(x) := x (W_r + W_p) = x W_r + x U_p \diag(S_p) V^\top_p. 
\end{align}


We adopt an orthonormal regularizer for $U_p$ and $V_p$.
\begin{definition}[Orthonormal regularizer] \label{def:reg}
The orthonormal regularizer is defined as
\begin{align*}
    \mathcal{L}_\reg(U_p,V_p) := \|U^\top_p U_p - I_m\|_F^2 + \|V^\top_p V_p - I_n\|_F^2.
\end{align*}
\end{definition}
The regularizer could enhance their orthonormality during training. We discuss and verify its importance and effectiveness in Section~\ref{sec:ana} and \ref{sec:theory}.

Therefore, parameter updating of $W_p$ in a $\method$ adapter at training step $t$ could be expressed as:
\begin{align}
W_{p,t+1} = & W_{p,t} - \eta_t \nabla_{W_{p,t}} \mathcal{L}_\train - \gamma_t \nabla_{W_{p,t}} \mathcal{L}_\reg. \label{eq:update_original}
\end{align}

At training step $t$, $\nabla_{W_{p,t}} \mathcal{L}_\train$ denotes the gradient of $\mathcal{L}_\train$ respect to $W_{p,t}$, and $\nabla_{W_{p,t}} \mathcal{L}_\reg$ denotes the gradient of the orthonormal regularizer loss $\mathcal{L}_\reg$ respect to $W_{p,t}$. $\eta_t$ and $\gamma_t$ are the learning rates for training loss and regularizer loss at step $t$, respectively.

We update the $\method$ as the following for implementation simplicity
\begin{align}
W_{p,t+1} = & W_{p,t} - \eta_t \left( \nabla_{W_{p,t}} \mathcal{L}_{\train} + \frac{\gamma}{\eta_d} \nabla_{W_{p,t}} \mathcal{L}_\reg \right), \label{eq:update}
\end{align}
$\eta_d$ is the maximum learning rate from the scheduler. This implementation allows us to use only one optimizer and scheduler to deal with two different learning rates separately.

%% file: 05_theory.tex
\section{Theoretical Analysis}
\label{sec:theory}

\subsection{Convergence Rate}

We begin by analyzing the convergence behavior of gradient descent when applied to our objective function, which consists of a data-fitting loss $L_\train$ and our orthonormal regularizer $\mathcal{L}_\reg$.

\begin{lemma}[Lipschitz continuity of $\mathcal{L}_\reg$]
\label{lem:lipschitz:squared}
Suppose
$\|U_p\|_F \leq M_U$ and $\|V_p\|_F \leq M_V$.
Then $\mathcal{L}_\reg$ is Lipschitz continuous in the Frobenius norm:
\begin{align*}
|\mathcal{L}_\reg(U_p^1,V_p^1)-\mathcal{L}_\reg(U_p^2,V_p^2)|
\leq
L_\reg(\|U_p^1-U_p^2\|_F+\|V_p^1-V_p^2\|_F),
\end{align*}
where
\begin{align*}
L_\reg
= 4 M_U (M_U^2+1) + 4 M_V (M_V^2+1).
\end{align*}
\end{lemma}

\begin{proof}
Compute the partial gradients
\begin{align*}
\nabla_{U_p} \mathcal{L}_\reg &= 4 U_p(U_p^\top U_p - I_m), \\
\nabla_{V_p} \mathcal{L}_\reg &= 4 V_p(V_p^\top V_p - I_n).
\end{align*}
Hence
\begin{align*}
\|\nabla \mathcal{L}_\reg\|_F
&\leq 4 \|U_p\| \|U_p^\top U_p - I_m\|_F + 4 \|V_p\| \|V_p^\top V_p - I_n\|_F \\
&\leq 4 M_U (M_U^2 + 1) + 4 M_V (M_V^2 + 1).
\end{align*}
By the mean value theorem for vector functions,
\begin{align*}
|\mathcal{L}_\reg(X) - \mathcal{L}_\reg(Y)| \leq \max_Z \|\nabla \mathcal{L}_\reg(Z)\|_F \|X - Y\|_F,
\end{align*}
and the claimed bound follows.
\end{proof}

We now make two standard assumptions to ensure well-behaved optimization.

\begin{assumption}[Smoothness and strong convexity of $L_\train$]
\label{asm:train}
The data term $L_\train(W_p)$ is twice differentiable, 
$\mu_\train$-strongly convex and $L_\train$-smooth:
\begin{align*}
\mu_\train I \preceq \nabla^2 L_\train(W) \preceq  L_\train I
\quad \text{for all } W_p.
\end{align*}
\end{assumption}

\begin{assumption}[Hessian lower bound for $\mathcal{L}_\reg$]
\label{asm:reg_hess}
There is a constant $C_\reg \geq 0$ such that
\begin{align*}
\nabla^2 \mathcal{L}_\reg(W) \succeq -C_\reg I
\quad \text{for all } W = (U_p,V_p).
\end{align*}
\end{assumption}

The next theorem establishes that, under these assumptions, $\method$ converges linearly.

\begin{theorem}[Linear convergence of $\method$]
\label{thm:rate}
Let
\begin{align*}
F(W_p) = L_\train(W_p) + \gamma \mathcal{L}_\reg(W_p),
\end{align*}
and suppose Assumptions \ref{asm:train} and \ref{asm:reg_hess} hold.  If
\begin{align*}
0 < \gamma < \frac{\mu_\train}{C_\reg},
\quad
\eta \in (0, \frac{2}{L_\train + \gamma L_\reg} ),
\end{align*}
then gradient descent
\begin{align*}
W_p^{t+1} = W_p^t - \eta \nabla F(W_p^t)
\end{align*}
satisfies
\begin{align*}
F(W_p^t) - F(W_p^*) \leq (1 - \eta (\mu_\train - \gamma C_\reg))^t (F(W_p^0) - F(W_p^*)).
\end{align*}
In particular, setting $\eta = 1/(L_\train + \gamma L_\reg)$ gives
\begin{align*}
F(W_p^t) - F(W_p^*) \leq (1 - \frac{\mu_\train - \gamma C_\reg}{L_\train + \gamma L_\reg} )^t (F(W_p^0) - F(W_p^*)).
\end{align*}
\end{theorem}

\begin{proof}
By Assumption \ref{asm:train}, $\nabla^2 L_\train \geq \mu_\train I$ and by Assumption \ref{asm:reg_hess}, $\nabla^2 (\gamma \mathcal{L}_\reg) \geq -\gamma C_\reg I$.
Hence,
\begin{align*}
\nabla^2 F = \nabla^2 L_\train + \gamma \nabla^2 \mathcal{L}_\reg \geq (\mu_\train - \gamma C_\reg) I,
\end{align*}
and also $\nabla^2 F \leq (L_\train + \gamma L_\reg) I$. The claimed rate follows from standard gradient descent guarantees.
\end{proof}

\subsection{Condition Number}

We now analyze how the regularizer in $\method$ helps maintain a smaller condition number for the weight matrix. A well-conditioned weight matrix is essential for stable optimization and good generalization.

We begin with a lemma that shows the singular values of the regularized weight matrix stay close to those of the unregularized one, provided the regularizer gradient is small.

\begin{lemma} \label{lem:sig}
Let 
\begin{align*}
W_p^{\unreg,t} = U_p^{\unreg,t} S_p^{\unreg,t} (V_p^{\unreg,t})^\top,
\quad
W_p^{\reg,t} = U_p^{\reg,t} S_p^{\reg,t} (V_p^{\reg,t})^\top
\end{align*}
be the outputs of one step of $\method$ at step $t$ with and without regularizer, respectively.

If $\|\nabla_{W_p} \mathcal{L}_\reg\|_F \leq \epsilon_{\nabla}$, then for each singular value $\sigma_i$,
\begin{align*}
(1 - \epsilon)\sigma_i^{\unreg,t} \leq \sigma_i^{\reg,t} \leq (1 + \epsilon)\sigma_i^{\unreg,t},
\end{align*}
where $\epsilon = \gamma \epsilon_{\nabla}$.
\end{lemma}

\begin{proof}
We have
\begin{align*}
W_p^\reg - W_p^{\unreg,t} = \gamma \nabla_{W_p} \mathcal{L}_\reg,
\quad
\|W_p^\reg - W_p^{\unreg,t}\|_F = \gamma \epsilon_{\nabla}.
\end{align*}
By Weyl’s inequality,
\begin{align*}
|\sigma_i^{\reg,t} - \sigma_i^{\unreg,t}| \leq \|W_p^{\reg,t} - W_p^{\unreg,t}\| \leq \|W_p^{\reg,t} - W_p^{\unreg,t}\|_F \leq \gamma \epsilon_{\nabla}.
\end{align*}
The last inequality follows directly.
\end{proof}

We now prove our main theorem: the condition number of the regularized weight matrix is strictly smaller than that of the unregularized one.

\begin{theorem} \label{thm:condition}
Under the setup of Lemma~\ref{lem:sig}, assume that $\nabla \mathcal{L}_\train$ is invariant for all $t > 0$.
Let the orthonormal regularizer be defined in Definition~\ref{def:reg}
Then for every iteration $t > 0$,
\begin{align*}
\kappa(W_p^{\reg,t})
<
\kappa(W_p^{\unreg,t}),
\end{align*}
where $\kappa$ is defined in Definition~\ref{def:kappa}.
\end{theorem}

\begin{proof}
We divide the proof into four steps to illustrate how regularization improves conditioning.

\paragraph{Step 1. Factor-wise bounds.}

For any factorization $W = U S V^\top$ with diagonal $S$,
\begin{align*}
\|W\| \leq \|U\| \|S\| \|V\|,
\quad
\|W^{-1}\| \leq \|V\| \|S^{-1}\| \|U^{-1}\|.
\end{align*}
Hence,
\begin{align*}
\kappa(W) \leq \kappa(U) \kappa(S) \kappa(V).
\end{align*}

\paragraph{Step 2. Singular value perturbation.}
According to Lemma~\ref{lem:sig},
\begin{align*}
|\sigma_i^{\reg,t} - \sigma_i^{\unreg,t}| \leq \epsilon_t,
\end{align*}
which implies
\begin{align*}
\kappa(S_p^{\reg,t}) \leq \frac{1 + \epsilon_t}{1 - \epsilon_t} \kappa(S_p^{\unreg,t}).
\end{align*}

\paragraph{Step 3. Orthonormal regularizer bounds factor condition numbers.}
By definition of $\mathcal{L}_\reg$ in Definition~\ref{def:reg}, and $\nabla \mathcal{L}_\train$ is invariant for all $t > 0$,
\begin{align*}
\kappa(U_p^{\reg,t}) < \kappa(U_p^{\unreg,t}),
\quad
\kappa(V_p^{\reg,t}) < \kappa(V_p^{\unreg,t}).
\end{align*}

\paragraph{Step 4: Combine bounds to compare condition numbers.}
By the above,
\begin{align*}
\kappa(W_p^{\reg,t})
\leq & ~
\kappa(U_p^{\reg,t}) \kappa(S_p^{\reg,t}) \kappa(V_p^{\reg,t}) \\
\leq & ~
\kappa(U_p^{\reg,t}) \kappa(V_p^{\reg,t}) \frac{1 + \epsilon_t}{1 - \epsilon_t} \kappa(S_p^{\unreg,t}),
\end{align*}
and
\begin{align*}
\kappa(W_p^{\unreg,t})
\geq
\kappa(U_p^{\unreg,t}) \kappa(S_p^{\unreg,t}) \kappa(V_p^{\unreg,t}).
\end{align*}
So,
\begin{align*}
\frac{\kappa(W_p^{\reg,t})}{\kappa(W_p^{\unreg,t})}
\leq
\frac{\kappa(U_p^{\reg,t}) \kappa(V_p^{\reg,t})}{\kappa(U_p^{\unreg,t}) \kappa(V_p^{\unreg,t})}
\cdot
\frac{1 + \epsilon_t}{1 - \epsilon_t}
< 1.
\end{align*}
Thus,
\begin{align*}
\kappa(W_p^{\reg,t}) < \kappa(W_p^{\unreg,t}),
\end{align*}
completing the proof.
\end{proof}

%% file: 06_exp.tex
\section{Experiments}
\label{sec:exp}


\subsection{Numerical Experiments}
\label{sec:ana}

\paragraph{Analysis Method}
We analyze the the deviation of singular values (denoted by $\Delta \Sigma$) and singular vectors (denoted by $\Delta D$) from pre-trained weights during updating.
Our analysis suggests a significant difference in singular values and vectors' stability and an updating pattern of fine-tuning, LoRA, and $\method$.

The singular value and vector variations between pre-trained weight $W_0 \in \R^{m \times n}$ and tuned weight $W_t \in \R^{m \times n}$, which $t$ denotes the training step, could be defined as follows
\begin{align}
  \Delta \Sigma_t = \frac{1}{k} \sum_{i=1}^k | \sigma_t^i - \sigma_0^i |,
\end{align}
where $\Delta \Sigma_t$ represents the singular value difference between $W_0$ and $W_t$ at training step $t$. $\sigma^i_t$ denotes the $i$-th element in diagonal of $\Sigma_t$, where $\Sigma_t$ is decomposed from $W_t$ by performing SVD, $k = \min(m,n)$,
\begin{align*}
\Delta U_{t,j} = | \langle u_t^j, u_0^j \rangle |, ~~~~ \Delta V^\top_{t,i} = | \langle v_t^i, v_0^i \rangle |, ~~~~ \Delta D_t = 1 - \frac{1}{2 k} \sum_{i=0}^{k} (\Delta U_{t,i} + \Delta V^\top_{t,i}).
\end{align*}
Here, $k = \min(m,n)$; $u_t^j$ denotes the $j$-th column vector of matrix $U_t$, and $v_t^i$ denotes the $i$-th row vector of matrix $V_t^\top$; $\Delta D_t \in (0, 1)$ represents variation of singular vectors between $W_0$ and $W_t$ at training step $t$; $U_t$ and $V_t^\top$ are decomposed from $W_t$ by performing SVD.

We adopt the analysis on Llama 2 7B \cite{tli+23} using the first 100K data of MetaMathQA \cite{yjs+23}. We test fine-tuning, LoRA, and $\method$ (with and without regularizer). See Section~\ref{app:ana} for training details of the analysis.

\paragraph{Analysis Result}

Based on our collected data, this section analyzes the results of different training methods: fine-tuning, LoRA, and $\method$. The analysis data is illustrated in Figure~\ref{fig:result}.

\begin{figure*}[!ht]
\centerline{
\includegraphics[width=0.45\columnwidth]{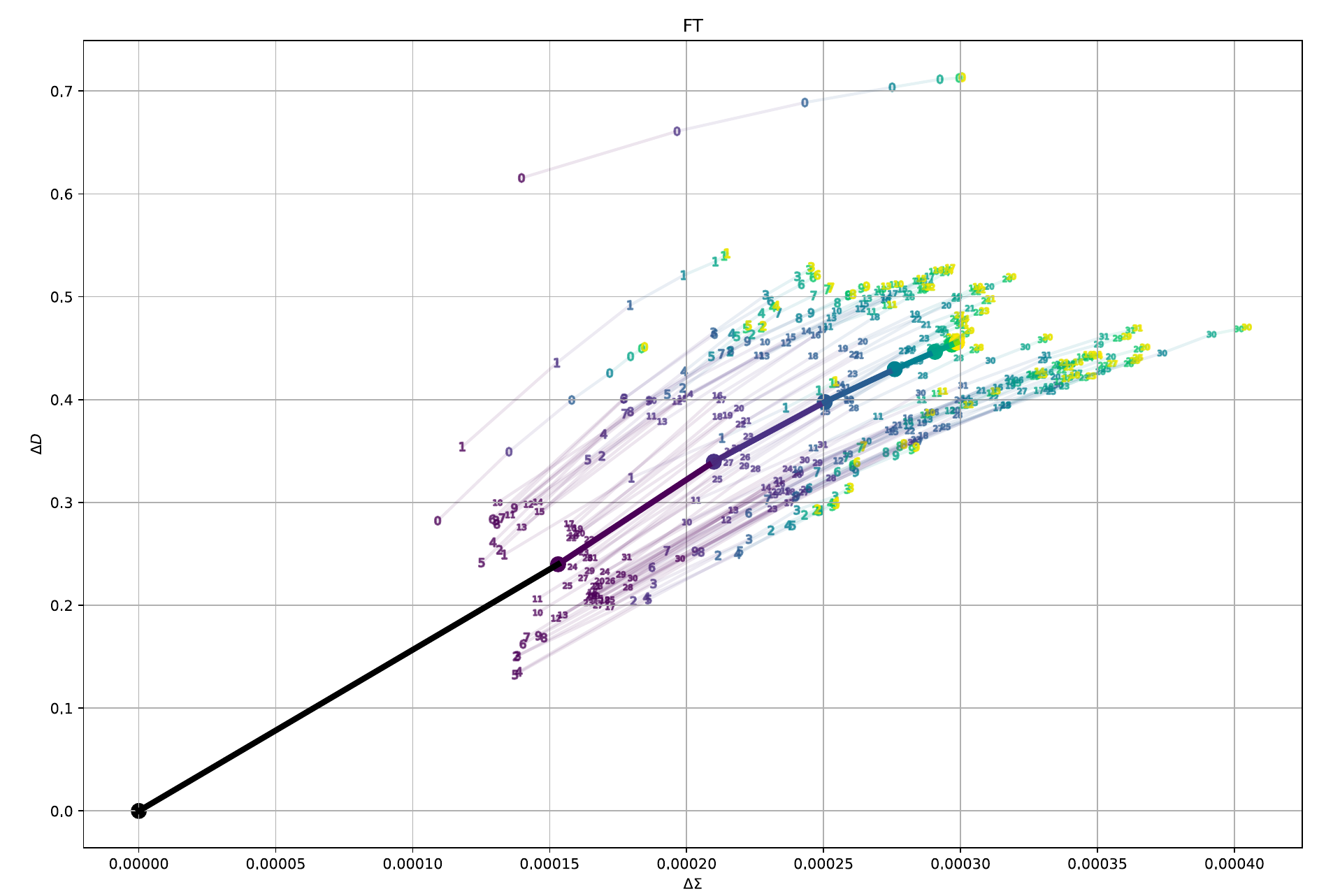}
\includegraphics[width=0.45\columnwidth]{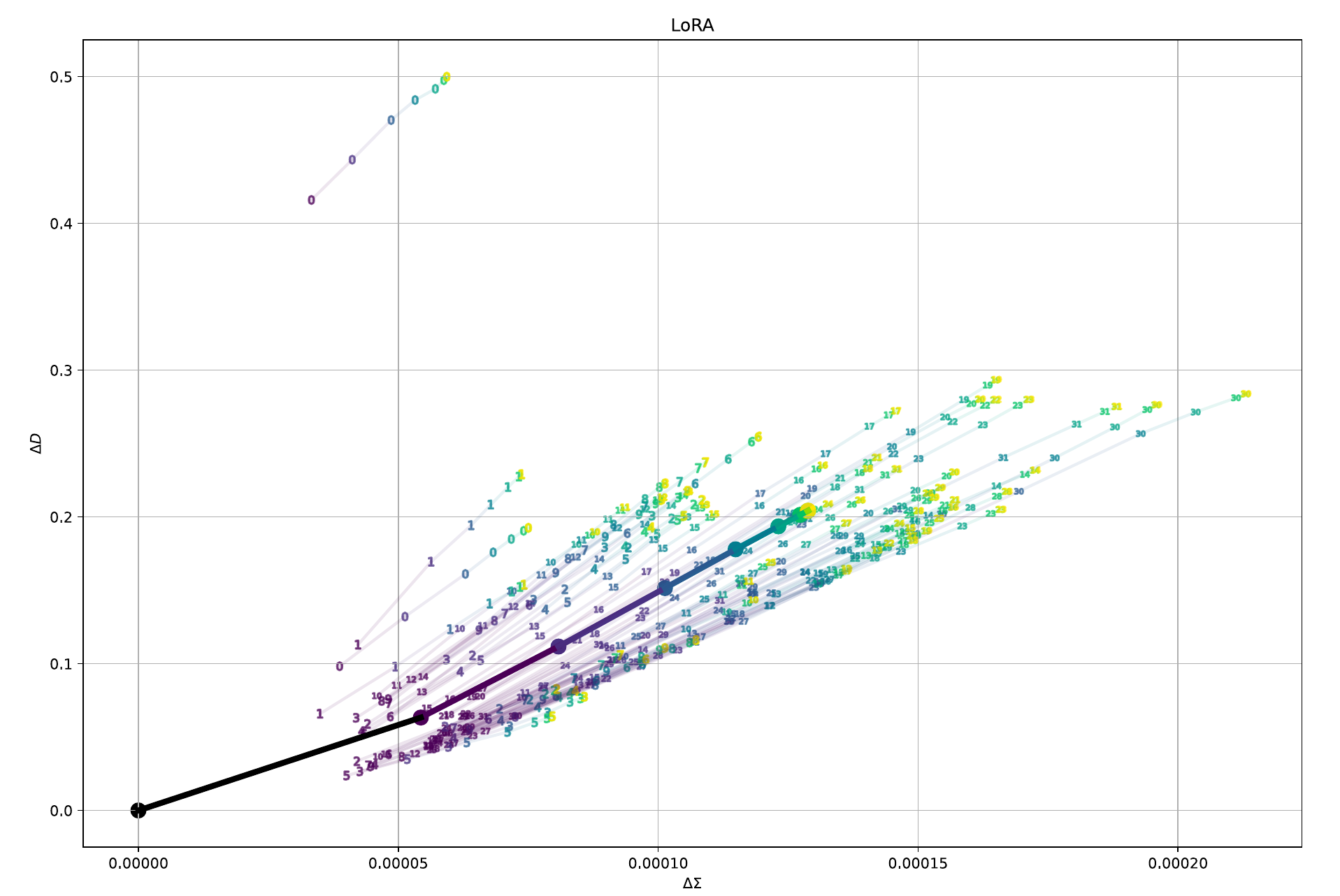}
}
\centerline{
\includegraphics[width=0.45\columnwidth]{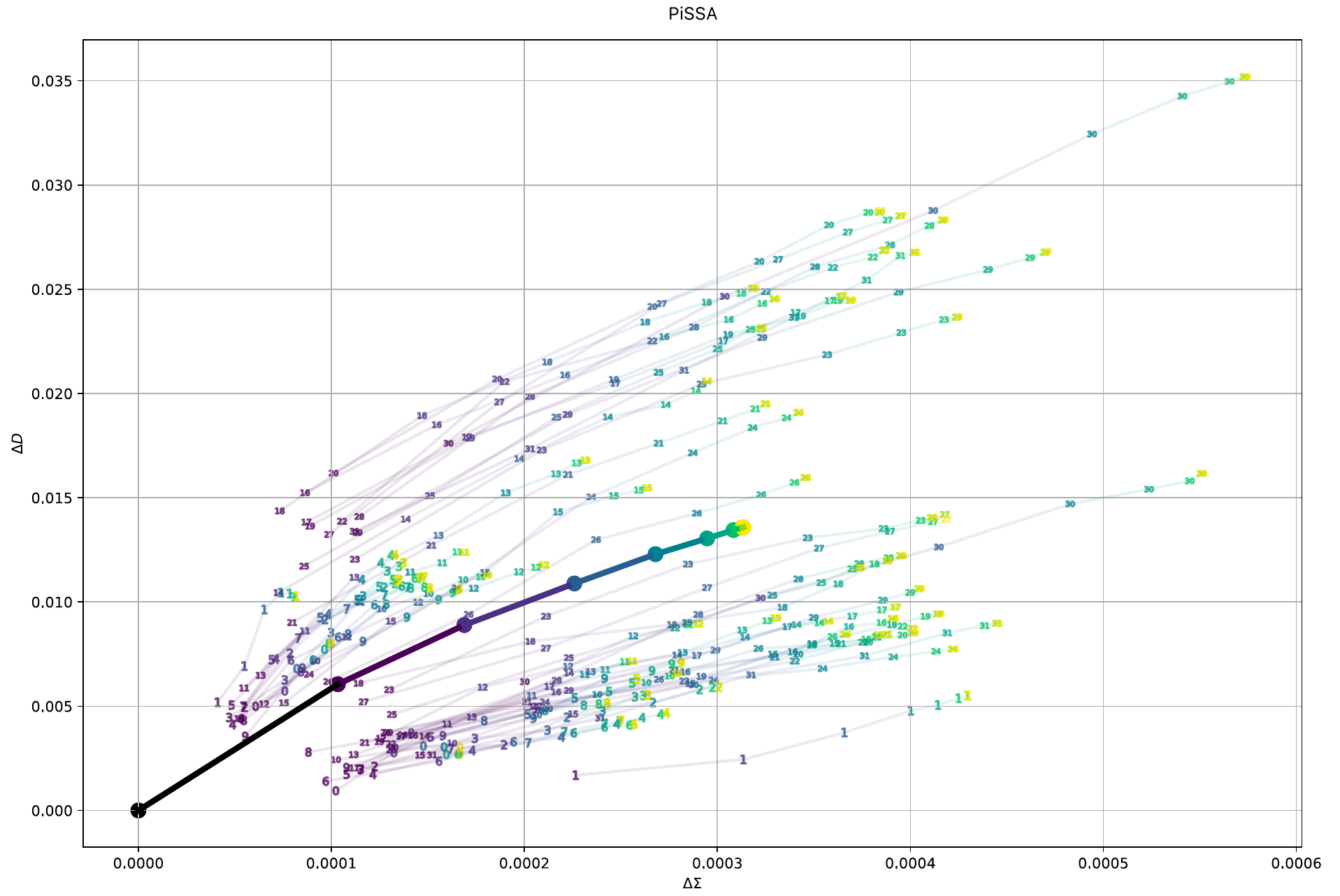}
\includegraphics[width=0.45\columnwidth]{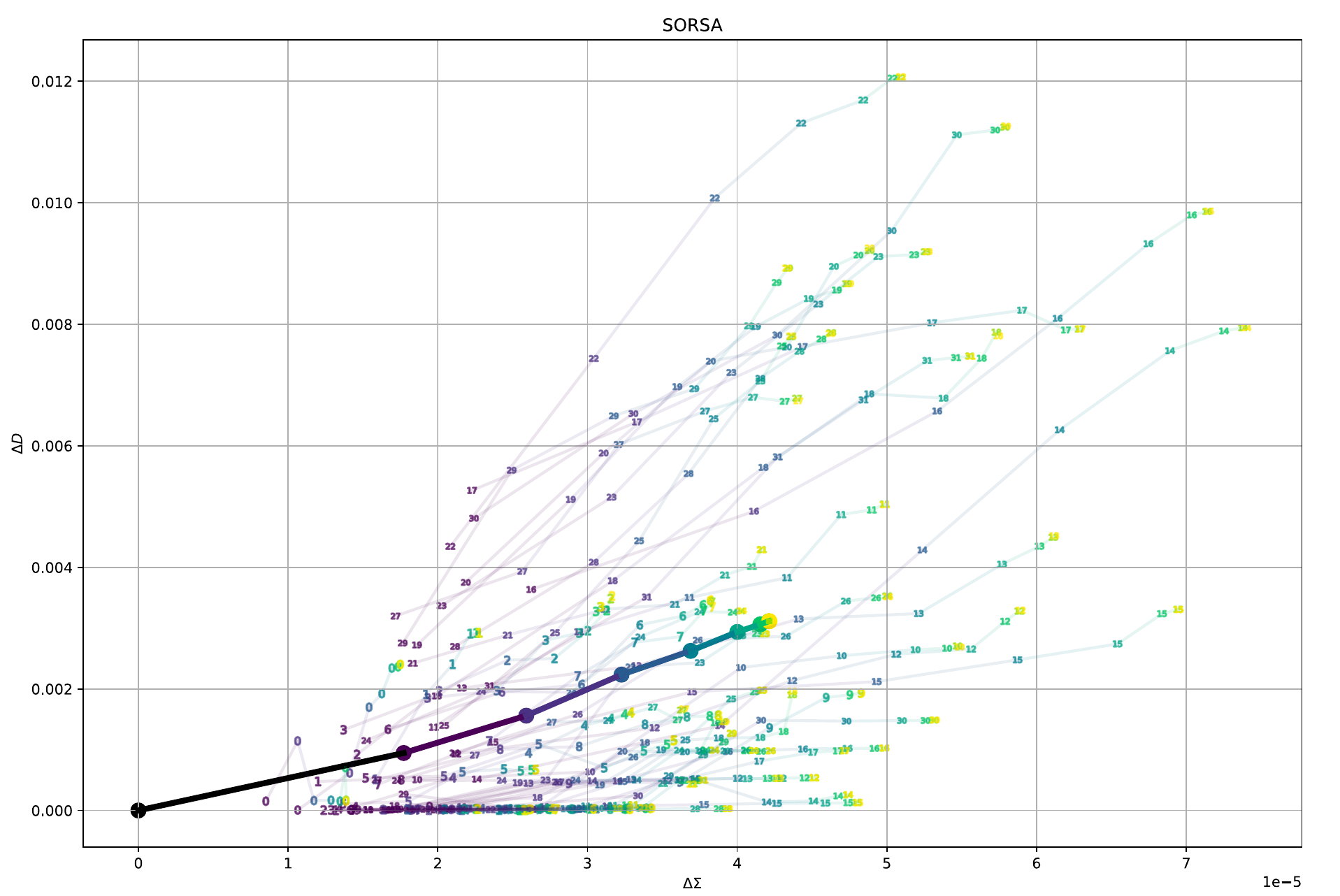}
}
\caption{\textbf{$\Delta D$ and $\Delta \Sigma$ of each trainable parameters during training steps.} Numbers in the plot represent layer of the weight. Dots represent mean $\Delta D$ and $\Delta \Sigma$ at specific step. Color from dark to light represent the time step from $0$ to $T$, where in these graphs $T=781$.}
\label{fig:result}
\end{figure*}

Based on our collected data, we analyze how different training methods - partial fine-tuning, LoRA, PiSSA and $\method$ - affect the pre-trained weights' structure and information preservation.

The analysis reveals several key insights about how these methods interact with the pre-trained knowledge:
\begin{itemize}
    \item Partial fine-tuning and LoRA show substantial alterations in singular vectors (large $\Delta D$), indicating significant disruption to the fundamental structure of the pre-trained weights. This extensive modification likely damages the model's carefully learned generalizations across multiple domains, leading to catastrophic forgetting. The parallel updating patterns across different layers suggest these methods make broad, potentially destructive changes throughout the model rather than targeted adaptations.

    \item $\method$ demonstrates significantly smaller changes in both singular values ($\Delta \Sigma$) and singular vectors ($\Delta D$) compared to other methods. This controlled modification suggests that $\method$ better preserves the pre-trained model's underlying knowledge structure while making precise adjustments for the downstream task. The orthonormal regularizer appears to act as a constraint that helps maintain the original geometric relationships within the weight matrices that encode the model's generalized capabilities.

    \item Different matrices in $\method$ show distinct, non-parallel updating patterns, unlike the uniform changes seen in other methods. This suggests $\method$ can identify and selectively modify the most relevant components for the target task while leaving other capabilities largely intact. This targeted adaptation explains why $\method$ can achieve better performance with less disruption to the model's general knowledge.

    \item We observe larger changes in both $\Delta D$ and $\Delta \Sigma$ for PiSSA compared to $\method$ along with updating patterns similar to LoRA and partial fine-tuning. Where the essential difference between PiSSA and $\method$ is the orthonormal regularizer, this empirically validates the regularizer's crucial role in preserving the pre-trained model's information structure while allowing efficient adaptation.
\end{itemize}

These patterns indicate $\method$'s ability to preserve the rich, generalized knowledge embedded in pre-trained weights while making precise adjustments for specific tasks. This property enables higher learning rates without over-fitting and explains $\method$'s improved performance across various benchmarks. The method's ability to maintain the model's fundamental structure while allowing targeted modifications represents a significant advance in efficient model adaptation.

\subsection{Empirical Experiments}

We conducted comparative experiments on different NLP tasks, including natural language generation (NLG) between $\method$, PiSSA \cite{mwz24}, LoRA \cite{hsw+22}, AdaLoRA \cite{zcb+23}, and full parameter fine-tuning.

\begin{figure*}[htbp]
\centerline{
\includegraphics[width=0.48\columnwidth]{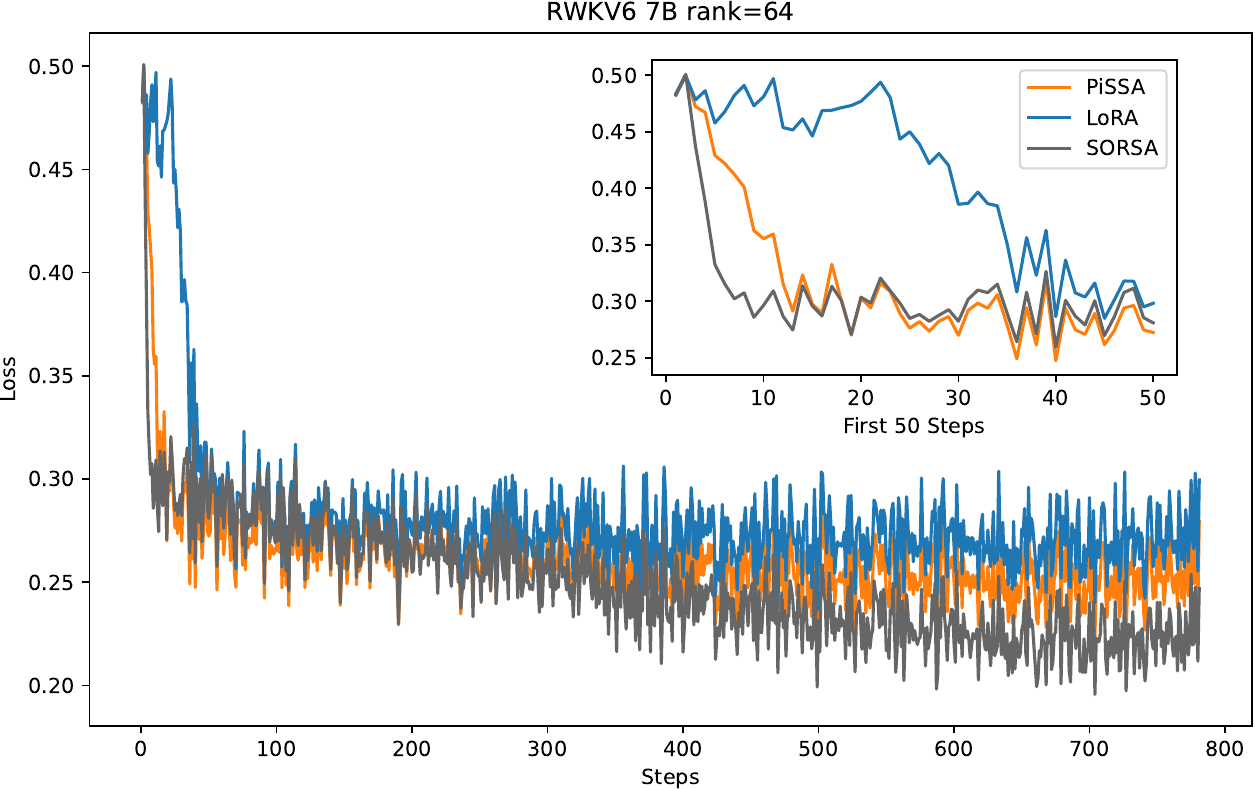}
\includegraphics[width=0.48\columnwidth]{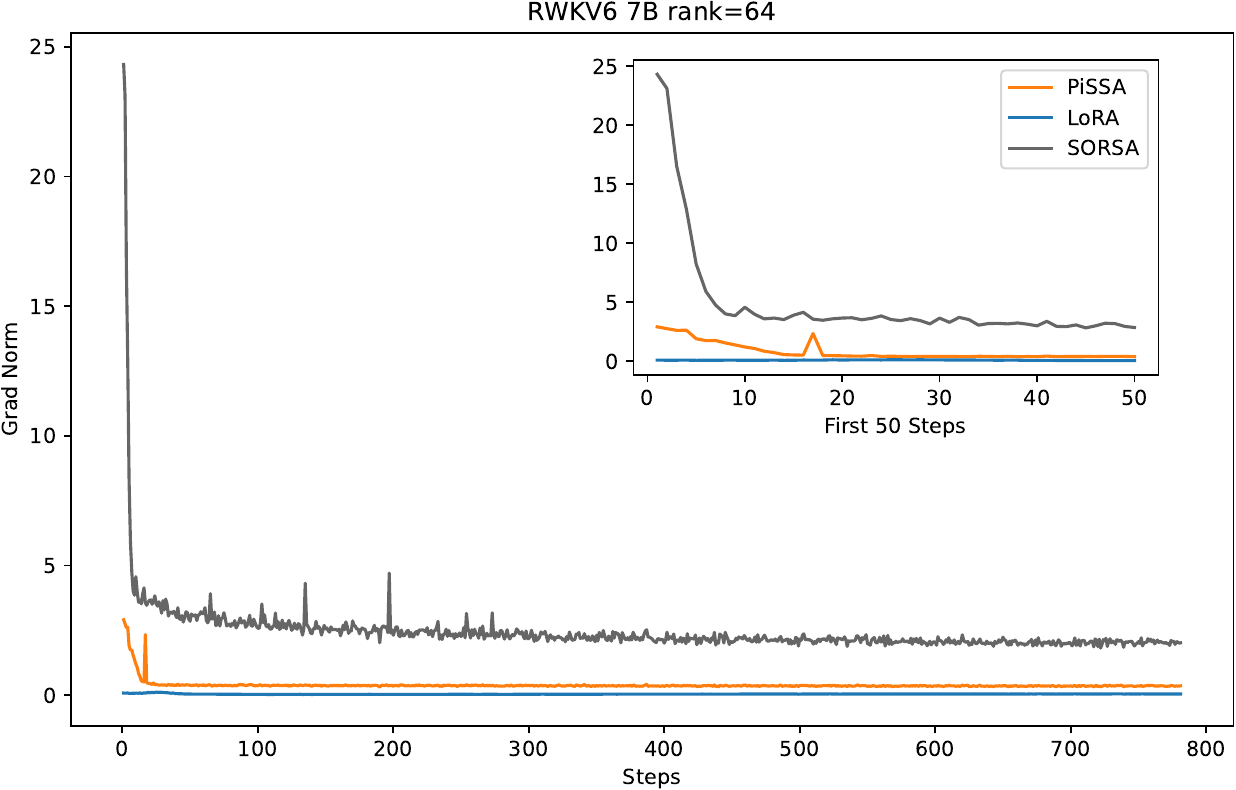}
}
\centerline{
\includegraphics[width=0.48\columnwidth]{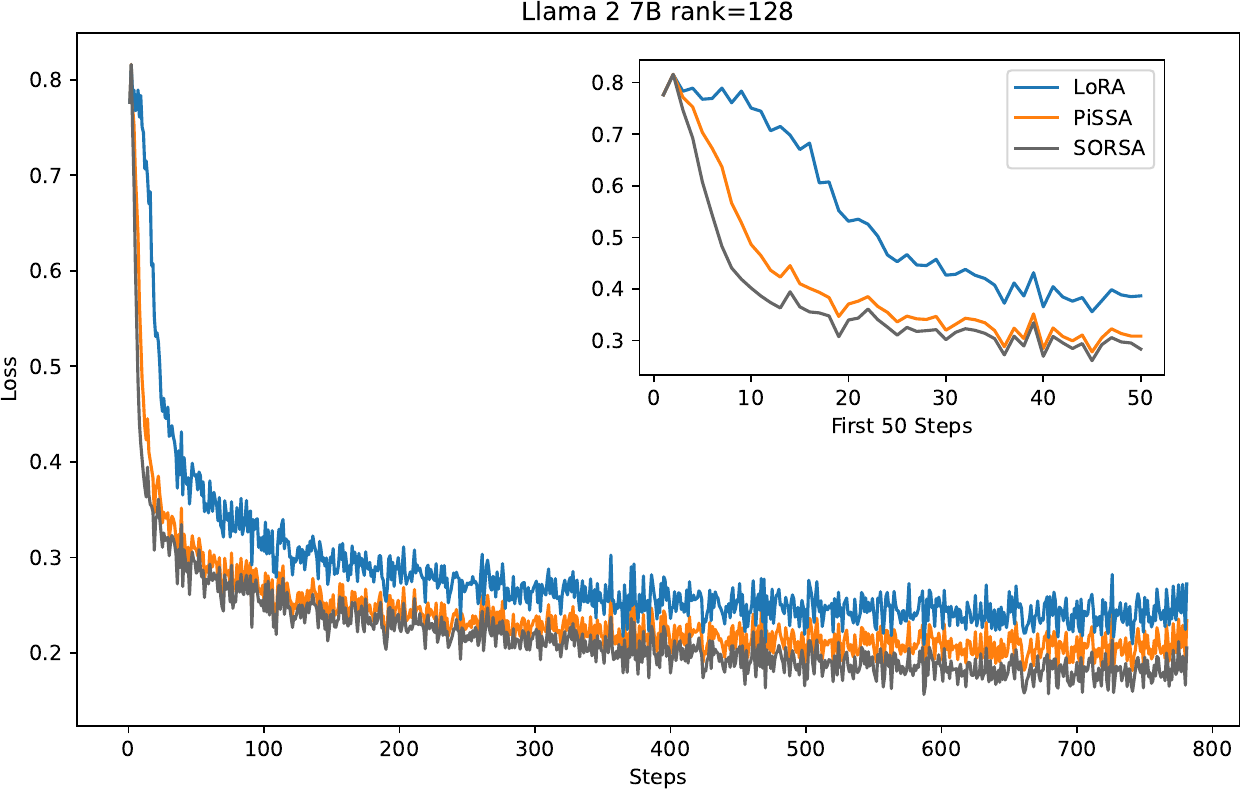}
\includegraphics[width=0.48\columnwidth]{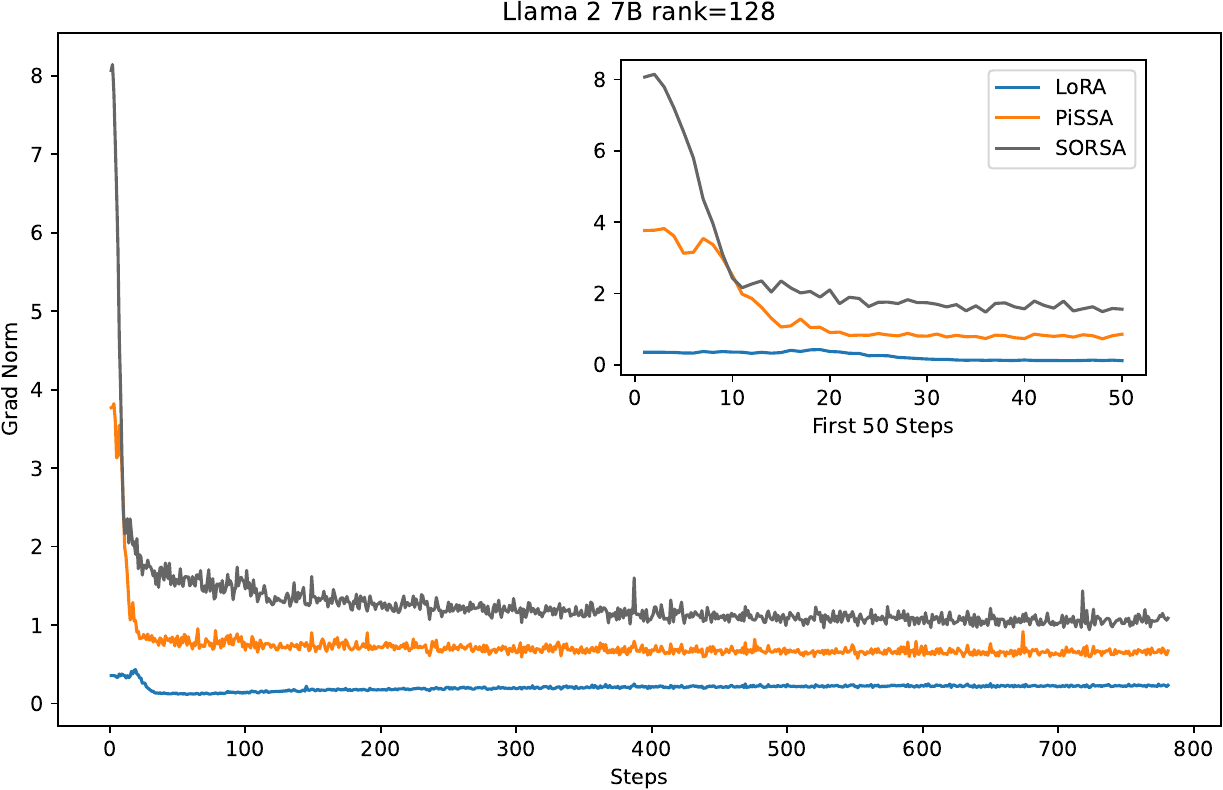}
}
\caption{The training loss and gradient norm comparison between $\method$, PiSSA, and LoRA on MetaMathQA training of RWKV6 7B and Llama 2 7B. LoRA and PiSSA curves of Llama 2 7B are from \cite{mwz24}.}
\label{fig:training}
\end{figure*}

We conducted NLG tests on Llama 2 7B \cite{tli+23}, RWKV6 7B \cite{pga+24}, Mistral 7B v0.1 \cite{jsb+23} and Gemma 7B \cite{gemmateam24}. We trained the models using the first 100K data in MetaMathQA \cite{yjs+23} and evaluated the model on GSM-8K \cite{ckb+21} and MATH \cite{hbk+21}. We also trained the model on the first 100K data in CodeFeedback Filtered Instruction \cite{zzs+24} dataset and evaluated it on HumanEval \cite{ctj+21}. The training process followed identical setups as the experiments conducted in PiSSA \cite{mwz24}. All reported values are accuracy in percentage. See Section~\ref{app:exp} for more details and hyperparameters of the training. We quoted some PiSSA, LoRA, and full parameter fine-tuning results from \cite{mwz24}. Some of our experiments were conducted on a single NVIDIA A100-SXM4 (80GB) GPU, and others were conducted on a single NVIDIA H100-SXM4 (80GB) GPU. See Table~\ref{tab:exp} for the results and Figure~\ref{fig:training} for the loss and gradient norm comparison.

\begin{table*}[!ht]
\caption{Comparing $\method$ with other methods on NLG tasks. $^\dagger$ denotes results from \cite{mwz24}. We use {\bf TPara.} to represent trainable parameters.}
\centering
\begin{tabular}{|l| l| l | l | l | l| l|} \hline
{\bf Model} & {\bf Method} &  {\bf TPara.} & {\bf GSM-8K} & {\bf MATH} & {\bf HumanEval}\\ \hline
Llama 2 7B & Full FT & 6738M & 49.05$^\dagger$ & 7.22$^\dagger$ & 21.34$^\dagger$ \\
Llama 2 7B & LoRA & 320M & 42.30$^\dagger$ & 5.50$^\dagger$ & 18.29$^\dagger$ \\
Llama 2 7B & PiSSA & 320M & \underline{53.07}$^\dagger$ & \underline{7.44}$^\dagger$ & \underline{21.95}$^\dagger$ \\ 
Llama 2 7B & AdaLoRA & 320M & 47.30 & 6.48 & 19.51 \\ 
Llama 2 7B & $\method$ & 320M & {\bf 56.03} & {\bf 10.36} & {\bf 24.39} \\ \hline
RWKV6 7B & LoRA & 176M & 8.04\footnotemark{} & 7.38 & 15.24 \\ 
RWKV6 7B  & PiSSA & 176M & 32.07 & \underline{9.42} & \underline{17.07} \\ 
RWKV6 7B  & AdaLoRA & 176M & \underline{33.28} & 8.08 & 15.85 \\ 
RWKV6 7B  & $\method$ & 176M & {\bf 45.87} & {\bf 11.32} & {\bf 22.56} \\ \hline
Mistral 7B & Full FT & 7242M & 67.02$^\dagger$ & 18.60$^\dagger$ & 45.12$^\dagger$ \\ 
Mistral 7B & LoRA & 168M & 67.70$^\dagger$ & 19.68$^\dagger$ & 43.90$^\dagger$\\ 
Mistral 7B & PiSSA & 168M & \underline{72.86}$^\dagger$ & \underline{21.54}$^\dagger$ & \underline{46.95}$^\dagger$\\ 
Mistral 7B & AdaLoRA & 168M & 72.25 & 21.06 & 45.73\\ 
Mistral 7B & $\method$ & 168M & {\bf 73.09} & {\bf 21.86}& {\bf 47.56}\\ \hline
Gemma 7B & Full FT & 8538M & 71.34$^\dagger$ & 22.74$^\dagger$ & 46.95$^\dagger$ \\ 
Gemma 7B & LoRA & 200M & 74.90$^\dagger$ & 31.28$^\dagger$ & 53.66$^\dagger$\\ 
Gemma 7B & PiSSA & 200M & 77.94$^\dagger$ & {\bf 31.94} $^\dagger$ & \underline{54.27}$^\dagger$\\ 
Gemma 7B & AdaLoRA & 200M & {\bf 78.99} & \underline{31.44} & {\bf 55.49} \\ 
Gemma 7B & $\method$ & 200M & \underline{78.09} & 29.52 & {\bf 55.49}\\ \hline
\end{tabular}
\label{tab:exp}
\end{table*}

The results showed that across all models tested, $\method$ generally outperformed other methods, though with some notable exceptions. For mathematical evaluations on Llama 2 7B, $\method$ scored 56.03\% on GSM-8K and 10.36\% on MATH, significantly outperforming other methods. For the RWKV6 7B model, $\method$ achieved 45.87\% accuracy on GSM-8K and 11.32\% on MATH, surpassing both PiSSA and AdaLoRA, with AdaLoRA showing competitive performance on GSM-8K at 33.28\%. On Mistral 7B, $\method$ reached 73.09\% on GSM-8K and 21.86\% on MATH, showing modest improvements over AdaLoRA's strong performance of 72.25\% and 21.06\%, respectively. With Gemma 7B, the results were mixed - while AdaLoRA achieved the highest GSM-8K score at 78.99\% and competitive MATH performance at 31.44\%, $\method$ maintained strong performance with 78.09\% on GSM-8K. However, its MATH score of 29.52\% was lower than other methods. In coding evaluations, $\method$ and AdaLoRA showed strong performance on HumanEval, with both methods achieving 55.49\% on Gemma 7B, while $\method$ maintained an edge across other model variants. Additionally, we did not include loss and gradient norm curves in our figure because the regularizer in AdaLoRA and Gaussian initialization caused significantly higher initial loss values, making direct comparisons with other methods inappropriate.

\footnotetext{This significant under-perform due to LoRA failed to learn the GSM-8K required answer formatting behavior.}
The Figure~\ref{fig:training} reveals that $\method$ and PiSSA exhibit nearly identical loss curves at the beginning and even slightly higher than PiSSA on RWKV-6 training. However, when the training step is approximately $t > 300$, $\method$ steadily decreases its loss. In contrast, LoRA and PiSSA show a deceleration in their loss reduction. The observations on loss curves are also valid for the changing rate of gradient norm, where $\method$ showed a more consistent decrease in gradient norm compared to LoRA and PiSSA. This can be explained by Theorem~\ref{thm:condition}, especially at later stages of training.

However, due to the limitation of computing resources, we only trained and benchmarked a small number of tasks.

%% file: 07_conclusion.tex
\section{Discussion and Conclusion}\label{sec:conclusion}

\label{sec:con}
In this paper, we introduced $\method$, a novel parameter-efficient fine-tuning (PEFT) method designed to enhance the adaptation of large language models (LLMs) for downstream tasks. $\method$ utilizes singular value decomposition (SVD) to split pre-trained weights into principal and residual components, only training the principal singular values and vectors while freezing the residuals. We implemented an orthonormal regularizer to maintain the orthonormality of singular vectors during training, ensuring efficient parameter updates and preserving the integrity of singular values.

Our experiments demonstrated that $\method$ outperforms existing PEFT methods, such as LoRA and PiSSA, in both convergence speed and accuracy on the NLG tasks. Specifically, Llama 2 7B, tuned with $\method$, achieved significant improvements in the GSM-8K and MATH benchmarks, highlighting the effectiveness of our approach.

We adopted singular values and vector analysis, comparing $\method$ with FT and LoRA. $\method$ is superior in preserving the pre-trained weight's singular values and vectors during training. This suggests an explanation for $\method$'s supreme performance demonstrated in the experiment. We also show the significance of the orthonormal regularizer through analysis.

Our gradient analysis provided a mathematical foundation for $\method$, demonstrating its convexity, Lipschitz continuity, and the crucial role of the regularizer in improving the optimization landscape. This theoretical framework explains $\method$'s empirical superior performance and offers valuable insights for future developments in adaptive learning algorithms.


Overall, $\method$ gives a new perspective on parameter-efficient fine-tuning, showcasing exceptional efficiency and robust performance. It outperforms existing methods like LoRA and PiSSA in several downstream tasks and maintains the practical benefits of low VRAM requirements, no inference latency, and ease of implementation. This innovative approach offers a promising direction of singular values and vector analysis for future research and practical applications in adapting pre-trained models, making it a pivotal development in the field.

%% file: _3_app.tex
\input{50_roadmap}
\input{51_exp_detail}
\input{52_impact}

%% file: 50_roadmap.tex
{\bf Roadmap.}
In the appendix, we present the experiments details in Section~\ref{app:exp}. Finally, we discuss the potential broader impact in Section~\ref{app:impact}.

%% file: 51_exp_detail.tex
\clearpage
\clearpage
\section{Experiments Details}
\label{app:exp}

\subsection{Analysis} \label{app:ana}
For the singular values and vectors analysis in Section~\ref{sec:ana}, we applied fine-tuning, LoRA and Conor (with and without orthonormal regularizer) on Llama 2 7B \cite{tli+23} model, training with the first 100K data in MetaMathQA \cite{yjs+23} dataset. We only calculated the loss on the response part. The models are trained with TF32 \& BF16 mix precision. See Table~\ref{tab:ana} for hyperparameters.

We used AdamW \cite{lh17} optimizer and cosine annealing scheduler in training. In the analysis, LoRA and Conor were only applied to \texttt{q\_proj} and \texttt{v\_proj} matrices. For FT, we set model's \texttt{q\_proj} and \texttt{v\_proj} matrices to trainable.

We also found we should only perform SVD for analysis using CPU, in order to get the precise analysis data.

\begin{table}[!ht]
\caption{Hyperparameters of training for the analysis}
\label{tab:ana}
\centering
\begin{tabular}{|l|c|c|c|c|} \hline
Model & Llama 2 7B & Llama 2 7B & Llama 2 7B & Llama 2 7B\\ \hline 
Method & FT & LoRA & Conor w/o reg & Conor\\ \hline
Mix-Precision & TF32+BF16 & TF32+BF16 & TF32+BF16 & TF32+BF16\\ 
Epoch & 1 & 1 & 1 & 1 \\ 
Batch Size & 128 & 128 & 128 & 128 \\ 
Max Length & 512 & 512 & 512 & 512 \\ 
Weight Decay & 0 & 0 & 0 & 0 \\ 
Warm-up Ratio & 0.03 & 0.03 & 0.03 & 0.03\\ 
Learning Rate & 2e-5 & 2e-5 & 2e-5 & 3e-5 \\ 
Grad Clip & False & False & False & False \\ 
Conor $\gamma$ & N/A & N/A & 0 & 5e-4 \\ 
Rank & N/A & 128 & 128 & 128\\ \hline
\end{tabular}
\end{table}

\subsection{NLG Experiments} 
For our NLG tasks, we adapted Llama 2 7B \cite{tli+23}, RWKV6 7B \cite{pga+24}, Mistral 7B v0.1 \cite{jsb+23} Gemma 7B \cite{gemmateam24} models by Conor. For GSM-8K \cite{ckb+21} and MATH \cite{hbk+21} evaluations, we trained those models with the first 100K data in MetaMathQA \cite{yjs+23} dataset. For HumanEval \cite{ctj+21} evaluation, we use the first 100K data in CodeFeedback Filtered Instruction \cite{zzs+24} dataset.

We used AdamW \cite{lh17} optimizer and cosine annealing scheduler in training. Conor adapters were applied on all linear matrices in every layer. We only calculated the loss on the response part. The models are loaded in FP32 and trained with TF32 \& BF16 mix precision. In our experiments, we selected a higher learning rate for Conor than other methods to counterbalance the negative effect of orthonormal regularizer on optimizing toward lower training loss. See Table~\ref{tab:exp-hyp:train} and \ref{tab:exp-hyp:eval} for hyperparameters. 

\begin{table}[!ht]
\caption{Hyperparameters for training with Conor, LoRA and PiSSA on different models for GSM-8K and MATH}
\label{tab:exp-hyp:train}
\centering
\begin{tabular}{|l|c|c|c|c|c|} \hline
Model & Llama 2 7B & RWKV6 7B & RWKV6 7B& Mistral 7B & Gemma 7B \\ \hline
Method & Conor & Conor & LoRA\&PiSSA & Conor & Conor \\ \hline
Mix-Precision & TF32+BF16 & TF32+BF16 & TF32+BF16 & TF32+BF16 & TF32+BF16 \\ 
Epoch & 1 & 1 & 1 & 1 & 1 \\ 
Batch Size & 128 & 128 & 128& 128 & 128 \\ 
Max Length & 512 & 512 & 512 & 512 & 512 \\ 
Weight Decay & 0 & 0 & 0& 0 & 0 \\ 
Warm-up Ratio & 0.03 & 0.03 & 0.03 & 0.03 & 0.03 \\ 
Learning Rate & 3e-5 & 3e-5 & 2e-5 & 3e-5 & 3e-5 \\ 
Grad Clip & 1.0 & 1.0 & 1.0 & 1.0 & 1.0 \\ 
Conor $\gamma$ & 4e-4 & 4e-4 & N/A & 4e-4 & 4e-4 \\ 
Rank & 128 & 64 & 64& 64 & 64 \\ \hline
\end{tabular}
\end{table}

\begin{table}[!ht]
\caption{Hyperparameters for evaluation with Conor, LoRA and PiSSA on different models for GSM-8K and MATH. ML denotes Max Length.}
\label{tab:exp-hyp:eval}
\centering
\begin{tabular}{|l|c|c|c|c|c|} \hline
Model & Llama 2 7B & RWKV6 7B & RWKV6 7B& Mistral 7B & Gemma 7B \\ \hline
Method & Conor & Conor & LoRA \& PiSSA & Conor & Conor \\ 
Precision & BF16 & FP32 & FP32 & BF16 & BF16 \\ 
Sampling & False & False & False & False & False \\ 
Top-P & 1.0 & 1.0 & 1.0 & 1.0 & 1.0 \\ 
ML for GSM-8K & 1024 & 1024 & 1024 & 1024 & 1024 \\ 
ML for MATH & 2048 & 2048 & 2048 & 2048 & 2048 \\ 
ML for HumanEval & 2048 & 2048 & 2048 & 2048 & 2048\\ \hline
\end{tabular}
\end{table}

\begin{table}[!ht]
\caption{Hyperparameters of training for with AdaLoRA on different models for GSM-8K and MATH}
\label{tab:exp-hyp-ada:train}
\centering
\begin{tabular}{|l|c|c|c|c|} \hline
Model & Llama 2 7B & Mistral 7B & Gemma 7B & RWKV6 7B\\ \hline
Method & AdaLoRA & AdaLoRA & AdaLoRA & AdaLoRA \\ \hline
Mix-Precision & TF32+BF16 & TF32+BF16 & TF32+BF16 & TF32+BF16 \\ 
Epoch & 1 & 1 & 1 & 1 \\ 
Batch Size & 128 & 128 & 128 & 128 \\ 
Max Length & 512 & 512 & 512 & 512 \\ 
Weight Decay & 0 & 0 & 0 & 0 \\ 
Warm-up Ratio & 0.03 & 0.03 & 0.03 & 0.03 \\ 
Learning Rate & 2e-5 & 2e-5 & 2e-5 & 2e-5 \\ 
Grad Clip & 1.0 & 1.0 & 1.0 & 1.0 \\ 
$\beta_1$ & 0.85 & 0.85 & 0.85 & 0.85 \\ 
$\beta_2$ & 0.85 & 0.85 & 0.85 & 0.85 \\ 
$r_{init}$ & 128 & 64 & 64 & 64 \\ 
$r_{target}$ & 128 & 64 & 64 & 64 \\ 
$t_{init}$ & 100 & 100 & 100 & 100 \\ 
$t_{final}$ & 600 & 600 & 600 & 600 \\ \hline
\end{tabular}
\end{table}

\begin{table}[!ht]
\caption{Hyperparameters of evaluation for with AdaLoRA on different models for GSM-8K and MATH. ML denotes Max Length.}
\label{tab:exp-hyp-ada:eval}
\centering
\begin{tabular}{|l|c|c|c|c|} \hline
Model & Llama 2 7B & Mistral 7B & Gemma 7B & RWKV6 7B\\ \hline
Method & AdaLoRA & AdaLoRA & AdaLoRA & AdaLoRA \\ \hline
Precision & BF16 & BF16 & BF16 & FP32\\ 
Sampling & False & False & False & False  \\ 
Top-P & 1.0 & 1.0 & 1.0 & 1.0 \\ 
ML for GSM-8K & 1024 & 1024 & 1024 & 1024\\ 
ML for MATH & 2048 & 2048 & 2048 & 2048\\ 
ML for HumanEval & 2048 & 2048 & 2048 & 2048\\ \hline
\end{tabular}
\end{table}








%% file: 52_impact.tex
\clearpage
\clearpage
\section{Broader Impact}
\label{app:impact}

In this paper, we introduce an innovative PEFT method in machine learning. Our approach significantly streamlined the model's tuning process, particularly for large-scale models, addressing both computational efficiency and environmental sustainability. As we push the boundaries of what is possible with Machine Learning, it is essential to consider the broader impacts of these advancements on the environment and ethical standards within the field.

Our experiments found that adapting with Conor could reduce VRAM consumption by up to 80\%. This significant reduction in hardware resource requirements also suggests less energy consumption than entire parameter fine-tuning methods. By enhancing efficiency, our approach could significantly reduce the carbon footprint of Machine Learning operations.


%% file: checklist.tex
\clearpage
\clearpage
\section*{NeurIPS Paper Checklist}

\begin{enumerate}

\item {\bf Claims}
    \item[] Question: Do the main claims made in the abstract and introduction accurately reflect the paper's contributions and scope?
    \item[] Answer: \answerYes{} 
    \item[] Justification: The abstract and introduction clearly and accurately introduce the core idea and contribution of this paper.

\item {\bf Limitations}
    \item[] Question: Does the paper discuss the limitations of the work performed by the authors?
    \item[] Answer: \answerYes{} 
    \item[] Justification: We admit that we only provide a small number of experiments due to limited computing resources in Section~\ref{sec:exp}.

\item {\bf Theory assumptions and proofs}
    \item[] Question: For each theoretical result, does the paper provide the full set of assumptions and a complete (and correct) proof?
    \item[] Answer: \answerYes{} 
    \item[] Justification: \justificationTODO{}

    \item {\bf Experimental result reproducibility}
    \item[] Question: Does the paper fully disclose all the information needed to reproduce the main experimental results of the paper to the extent that it affects the main claims and/or conclusions of the paper (regardless of whether the code and data are provided or not)?
    \item[] Answer: \answerYes{} 
    \item[] Justification: We provide detailed experiment setups in Section~\ref{app:exp}, and we provide the code in the supplemental materials.

\item {\bf Open access to data and code}
    \item[] Question: Does the paper provide open access to the data and code, with sufficient instructions to faithfully reproduce the main experimental results, as described in supplemental material?
    \item[] Answer: \answerYes{} 
    \item[] Justification: We only use open-source datasets and models to conduct the experiments. The code for training and evaluating is available in the supplemental materials. We will also release the code on GitHub once the paper is accepted.

\item {\bf Experimental setting/details}
    \item[] Question: Does the paper specify all the training and test details (e.g., data splits, hyperparameters, how they were chosen, type of optimizer, etc.) necessary to understand the results?
    \item[] Answer: \answerYes{} 
    \item[] Justification: We provide all details in Section~\ref{app:exp}.

\item {\bf Experiment statistical significance}
    \item[] Question: Does the paper report error bars suitably and correctly defined or other appropriate information about the statistical significance of the experiments?
    \item[] Answer: \answerYes{} 
    \item[] Justification: 
    The experiments are conducted by multi-run. The error bar is too small to be reported.

\item {\bf Experiments compute resources}
    \item[] Question: For each experiment, does the paper provide sufficient information on the computer resources (type of compute workers, memory, time of execution) needed to reproduce the experiments?
    \item[] Answer: \answerYes{} 
    \item[] Justification: We provide the resource we use in Section~\ref{sec:exp}.
    
\item {\bf Code of ethics}
    \item[] Question: Does the research conducted in the paper conform, in every respect, with the NeurIPS Code of Ethics \url{https://neurips.cc/public/EthicsGuidelines}?
    \item[] Answer: \answerYes{} 
    \item[] Justification: This paper conform the NeurIPS Code of Ethics.

\item {\bf Broader impacts}
    \item[] Question: Does the paper discuss both potential positive societal impacts and negative societal impacts of the work performed?
    \item[] Answer: \answerYes{} 
    \item[] Justification: We discussed potential impacts in Section~\ref{app:impact}.
    
\item {\bf Safeguards}
    \item[] Question: Does the paper describe safeguards that have been put in place for responsible release of data or models that have a high risk for misuse (e.g., pretrained language models, image generators, or scraped datasets)?
    \item[] Answer: \answerNA{} 
    \item[] Justification: This paper poses no such risks. We do not release any model or data.

\item {\bf Licenses for existing assets}
    \item[] Question: Are the creators or original owners of assets (e.g., code, data, models), used in the paper, properly credited and are the license and terms of use explicitly mentioned and properly respected?
    \item[] Answer: \answerYes{} 
    \item[] Justification: We follow all the license and cite all the assets properly.
    We conducted NLG tests on Llama 2 7B \cite{tli+23}, RWKV6 7B \cite{pga+24}, Mistral 7B v0.1 \cite{jsb+23} and Gemma 7B \cite{gemmateam24}. We trained the models using the first 100K data in MetaMathQA \cite{yjs+23} and evaluated the model on GSM-8K \cite{ckb+21} and MATH \cite{hbk+21}. We also trained the model on the first 100K data in CodeFeedback Filtered Instruction \cite{zzs+24} dataset and evaluated it on HumanEval \cite{ctj+21}.

\item {\bf New assets}
    \item[] Question: Are new assets introduced in the paper well documented and is the documentation provided alongside the assets?
    \item[] Answer: \answerNA{} 
    \item[] Justification: The paper does not release new assets.

\item {\bf Crowdsourcing and research with human subjects}
    \item[] Question: For crowdsourcing experiments and research with human subjects, does the paper include the full text of instructions given to participants and screenshots, if applicable, as well as details about compensation (if any)? 
    \item[] Answer: \answerNA{} 
    \item[] Justification: The paper does not involve crowdsourcing nor research with human subjects.

\item {\bf Institutional review board (IRB) approvals or equivalent for research with human subjects}
    \item[] Question: Does the paper describe potential risks incurred by study participants, whether such risks were disclosed to the subjects, and whether Institutional Review Board (IRB) approvals (or an equivalent approval/review based on the requirements of your country or institution) were obtained?
    \item[] Answer: \answerNA{} 
    \item[] Justification: The paper does not involve crowdsourcing nor research with human subjects.

\item {\bf Declaration of LLM usage}
    \item[] Question: Does the paper describe the usage of LLMs if it is an important, original, or non-standard component of the core methods in this research? Note that if the LLM is used only for writing, editing, or formatting purposes and does not impact the core methodology, scientific rigorousness, or originality of the research, declaration is not required.
    \item[] Answer: \answerNA{} 
    \item[] Justification: We only use LLM to polish our writing.

\end{enumerate}

%% file: main.bbl
\newcommand{\etalchar}[1]{$^{#1}$}
\begin{thebibliography}{LCBH{\etalchar{+}}22}

\bibitem[AAA{\etalchar{+}}23]{aaa+23}
Josh Achiam, Steven Adler, Sandhini Agarwal, Lama Ahmad, Ilge Akkaya, Florencia~Leoni Aleman, Diogo Almeida, Janko Altenschmidt, Sam Altman, Shyamal Anadkat, et~al.
\newblock Gpt-4 technical report.
\newblock {\em arXiv preprint arXiv:2303.08774}, 2023.

\bibitem[BMR{\etalchar{+}}20]{bmr+20}
Tom Brown, Benjamin Mann, Nick Ryder, Melanie Subbiah, Jared~D Kaplan, Prafulla Dhariwal, Arvind Neelakantan, Pranav Shyam, Girish Sastry, Amanda Askell, et~al.
\newblock Language models are few-shot learners.
\newblock {\em Advances in neural information processing systems}, 33:1877--1901, 2020.

\bibitem[B{\"u}y24]{b24}
Kerim B{\"u}y{\"u}kaky{\"u}z.
\newblock Olora: Orthonormal low-rank adaptation of large language models.
\newblock {\em arXiv preprint arXiv:2406.01775}, 2024.

\bibitem[CCL{\etalchar{+}}25]{ccl+25}
Yang Cao, Bo~Chen, Xiaoyu Li, Yingyu Liang, Zhizhou Sha, Zhenmei Shi, Zhao Song, and Mingda Wan.
\newblock Force matching with relativistic constraints: A physics-inspired approach to stable and efficient generative modeling.
\newblock {\em arXiv preprint arXiv:2502.08150}, 2025.

\bibitem[CGH{\etalchar{+}}25]{cgh+25}
Yuefan Cao, Xuyang Guo, Jiayan Huo, Yingyu Liang, Zhenmei Shi, Zhao Song, Jiahao Zhang, and Zhen Zhuang.
\newblock Text-to-image diffusion models cannot count, and prompt refinement cannot help.
\newblock {\em arXiv preprint arXiv:2503.06884}, 2025.

\bibitem[CGL{\etalchar{+}}25]{cgl+25}
Bo~Chen, Chengyue Gong, Xiaoyu Li, Yingyu Liang, Zhizhou Sha, Zhenmei Shi, Zhao Song, and Mingda Wan.
\newblock High-order matching for one-step shortcut diffusion models.
\newblock {\em arXiv preprint arXiv:2502.00688}, 2025.

\bibitem[CHL{\etalchar{+}}24]{chl+24_gat}
Ya-Ting Chang, Zhibo Hu, Xiaoyu Li, Shuiqiao Yang, Jiaojiao Jiang, and Nan Sun.
\newblock Dihan: A novel dynamic hierarchical graph attention network for fake news detection.
\newblock In {\em Proceedings of the 33rd ACM International Conference on Information and Knowledge Management}, pages 197--206, 2024.

\bibitem[CKB{\etalchar{+}}21]{ckb+21}
Karl Cobbe, Vineet Kosaraju, Mohammad Bavarian, Mark Chen, Heewoo Jun, Lukasz Kaiser, Matthias Plappert, Jerry Tworek, Jacob Hilton, Reiichiro Nakano, et~al.
\newblock Training verifiers to solve math word problems.
\newblock {\em arXiv preprint arXiv:2110.14168}, 2021.

\bibitem[CPM23]{cpm23}
Eli Chien, Chao Pan, and Olgica Milenkovic.
\newblock Efficient model updates for approximate unlearning of graph-structured data.
\newblock In {\em The Eleventh International Conference on Learning Representations}, 2023.

\bibitem[CTJ{\etalchar{+}}21]{ctj+21}
Mark Chen, Jerry Tworek, Heewoo Jun, Qiming Yuan, Henrique Ponde De~Oliveira Pinto, Jared Kaplan, Harri Edwards, Yuri Burda, Nicholas Joseph, Greg Brockman, et~al.
\newblock Evaluating large language models trained on code.
\newblock {\em arXiv preprint arXiv:2107.03374}, 2021.

\bibitem[DBK{\etalchar{+}}20]{dbk+20}
Alexey Dosovitskiy, Lucas Beyer, Alexander Kolesnikov, Dirk Weissenborn, Xiaohua Zhai, Thomas Unterthiner, Mostafa Dehghani, Matthias Minderer, Georg Heigold, Sylvain Gelly, et~al.
\newblock An image is worth 16x16 words: Transformers for image recognition at scale.
\newblock {\em arXiv preprint arXiv:2010.11929}, 2020.

\bibitem[DLFZ22]{dlf+21}
Chenhui Deng, Xiuyu Li, Zhuo Feng, and Zhiru Zhang.
\newblock {GARNET}: Reduced-rank topology learning for robust and scalable graph neural networks.
\newblock In {\em The First Learning on Graphs Conference}, 2022.

\bibitem[DPNT23]{dpnt23}
Quan Dao, Hao Phung, Binh Nguyen, and Anh Tran.
\newblock Flow matching in latent space.
\newblock {\em arXiv preprint arXiv:2307.08698}, 2023.

\bibitem[DZM{\etalchar{+}}21]{dzm+21}
Xiaohan Ding, Xiangyu Zhang, Ningning Ma, Jungong Han, Guiguang Ding, and Jian Sun.
\newblock Repvgg: Making vgg-style convnets great again.
\newblock In {\em Proceedings of the IEEE/CVF conference on computer vision and pattern recognition}, pages 13733--13742, 2021.

\bibitem[FHLA24]{fhla24}
Kevin Frans, Danijar Hafner, Sergey Levine, and Pieter Abbeel.
\newblock One step diffusion via shortcut models.
\newblock {\em arXiv preprint arXiv:2410.12557}, 2024.

\bibitem[FLGW25]{flgw25}
Yangqi Feng, Shing-Ho~J Lin, Baoyuan Gao, and Xian Wei.
\newblock Lipschitz constant meets condition number: Learning robust and compact deep neural networks.
\newblock {\em arXiv preprint arXiv:2503.20454}, 2025.

\bibitem[FLZ{\etalchar{+}}21]{flz+21}
Ziwang Fu, Feng Liu, Jiahao Zhang, Hanyang Wang, Chengyi Yang, Qing Xu, Jiayin Qi, Xiangling Fu, and Aimin Zhou.
\newblock Sagn: semantic adaptive graph network for skeleton-based human action recognition.
\newblock In {\em Proceedings of the 2021 International Conference on Multimedia Retrieval}, pages 110--117, 2021.

\bibitem[GDJ{\etalchar{+}}24]{gdj+24}
Aaron Grattafiori, Abhimanyu Dubey, Abhinav Jauhri, Abhinav Pandey, Abhishek Kadian, Ahmad Al-Dahle, Aiesha Letman, Akhil Mathur, Alan Schelten, Alex Vaughan, et~al.
\newblock The llama 3 herd of models.
\newblock {\em arXiv preprint arXiv:2407.21783}, 2024.

\bibitem[{Gem}24]{gemmateam24}
{Google DeepMind} {Gemma Team}.
\newblock Gemma: Open models based on gemini research and technology.
\newblock {\em arXiv preprint arXiv:2403.08295}, 2024.

\bibitem[GHH{\etalchar{+}}25]{ghh+25}
Xuyang Guo, Zekai Huang, Jiayan Huo, Yingyu Liang, Zhenmei Shi, Zhao Song, and Jiahao Zhang.
\newblock Can you count to nine? a human evaluation benchmark for counting limits in modern text-to-video models.
\newblock {\em arXiv preprint arXiv:2504.04051}, 2025.

\bibitem[GHS{\etalchar{+}}25a]{ghs+25_physical}
Xuyang Guo, Jiayan Huo, Zhenmei Shi, Zhao Song, Jiahao Zhang, and Jiale Zhao.
\newblock T2vphysbench: A first-principles benchmark for physical consistency in text-to-video generation.
\newblock {\em arXiv preprint arXiv:2505.00337}, 2025.

\bibitem[GHS{\etalchar{+}}25b]{ghs+25_text}
Xuyang Guo, Jiayan Huo, Zhenmei Shi, Zhao Song, Jiahao Zhang, and Jiale Zhao.
\newblock T2vtextbench: A human evaluation benchmark for textual control in video generation models.
\newblock {\em arXiv preprint arXiv:2505.04946}, 2025.

\bibitem[GPPX24]{gpp+24}
Shivam Gupta, Aditya Parulekar, Eric Price, and Zhiyang Xun.
\newblock Improved sample complexity bounds for diffusion model training.
\newblock {\em Advances in Neural Information Processing Systems}, 37:40976--41012, 2024.

\bibitem[GSS{\etalchar{+}}21]{gss+21}
Simon Geisler, Tobias Schmidt, Hakan Sirin, Daniel Z\"ugner, Aleksandar Bojchevski, and Stephan G\"unnemann.
\newblock Robustness of graph neural networks at scale.
\newblock In {\em NeurIPS}, 2021.

\bibitem[HBK{\etalchar{+}}21]{hbk+21}
Dan Hendrycks, Collin Burns, Saurav Kadavath, Akul Arora, Steven Basart, Eric Tang, Dawn Song, and Jacob Steinhardt.
\newblock Measuring mathematical problem solving with the math dataset.
\newblock {\em arXiv preprint arXiv:2103.03874}, 2021.

\bibitem[HGJ{\etalchar{+}}19]{hgj+19}
Neil Houlsby, Andrei Giurgiu, Stanislaw Jastrzebski, Bruna Morrone, Quentin De~Laroussilhe, Andrea Gesmundo, Mona Attariyan, and Sylvain Gelly.
\newblock Parameter-efficient transfer learning for nlp.
\newblock In {\em International conference on machine learning}, pages 2790--2799. PMLR, 2019.

\bibitem[HJA20]{hja20}
Jonathan Ho, Ajay Jain, and Pieter Abbeel.
\newblock Denoising diffusion probabilistic models.
\newblock {\em Advances in neural information processing systems}, 33:6840--6851, 2020.

\bibitem[HJW22]{hjw22}
Zhimeng Han, Muwei Jian, and Gai-Ge Wang.
\newblock Convunext: An efficient convolution neural network for medical image segmentation.
\newblock {\em Knowledge-based systems}, 253:109512, 2022.

\bibitem[HSW{\etalchar{+}}22]{hsw+22}
Edward~J Hu, Yelong Shen, Phillip Wallis, Zeyuan Allen-Zhu, Yuanzhi Li, Shean Wang, Lu~Wang, Weizhu Chen, et~al.
\newblock Lora: Low-rank adaptation of large language models.
\newblock {\em ICLR}, 1(2):3, 2022.

\bibitem[HWL{\etalchar{+}}24]{hwl+24}
Jerry Yao-Chieh Hu, Weimin Wu, Zhuoru Li, Sophia Pi, Zhao Song, and Han Liu.
\newblock On statistical rates and provably efficient criteria of latent diffusion transformers (dits).
\newblock {\em Advances in Neural Information Processing Systems}, 37:31562--31628, 2024.

\bibitem[HZL{\etalchar{+}}23]{hzl+23}
Xiaotian Han, Tong Zhao, Yozen Liu, Xia Hu, and Neil Shah.
\newblock {MLPI}nit: Embarrassingly simple {GNN} training acceleration with {MLP} initialization.
\newblock In {\em The Eleventh International Conference on Learning Representations}, 2023.

\bibitem[HZM{\etalchar{+}}21]{hzm+21}
Junxian He, Chunting Zhou, Xuezhe Ma, Taylor Berg-Kirkpatrick, and Graham Neubig.
\newblock Towards a unified view of parameter-efficient transfer learning.
\newblock {\em arXiv preprint arXiv:2110.04366}, 2021.

\bibitem[HZRS16]{hzrs16}
Kaiming He, Xiangyu Zhang, Shaoqing Ren, and Jian Sun.
\newblock Deep residual learning for image recognition.
\newblock In {\em Proceedings of the IEEE conference on computer vision and pattern recognition}, pages 770--778, 2016.

\bibitem[JCX{\etalchar{+}}23]{jcx+23}
Bo~Jiang, Shaoyu Chen, Qing Xu, Bencheng Liao, Jiajie Chen, Helong Zhou, Qian Zhang, Wenyu Liu, Chang Huang, and Xinggang Wang.
\newblock Vad: Vectorized scene representation for efficient autonomous driving.
\newblock In {\em Proceedings of the IEEE/CVF International Conference on Computer Vision}, pages 8340--8350, 2023.

\bibitem[JSM{\etalchar{+}}23]{jsb+23}
Albert~Q Jiang, A~Sablayrolles, A~Mensch, C~Bamford, DS~Chaplot, D~de~las Casas, F~Bressand, G~Lengyel, G~Lample, L~Saulnier, et~al.
\newblock Mistral 7b.
\newblock {\em arXiv preprint arXiv:2310.06825}, 2023.

\bibitem[KLL{\etalchar{+}}25]{kll+25}
Yekun Ke, Xiaoyu Li, Yingyu Liang, Zhenmei Shi, and Zhao Song.
\newblock Circuit complexity bounds for visual autoregressive model.
\newblock {\em arXiv preprint arXiv:2501.04299}, 2025.

\bibitem[LARC21]{lac21}
Brian Lester, Rami Al-Rfou, and Noah Constant.
\newblock The power of scale for parameter-efficient prompt tuning.
\newblock {\em arXiv preprint arXiv:2104.08691}, 2021.

\bibitem[LCBH{\etalchar{+}}22]{lcb+22}
Yaron Lipman, Ricky~TQ Chen, Heli Ben-Hamu, Maximilian Nickel, and Matt Le.
\newblock Flow matching for generative modeling.
\newblock {\em arXiv preprint arXiv:2210.02747}, 2022.

\bibitem[LCC{\etalchar{+}}21]{lcc+21}
Maosen Li, Siheng Chen, Xu~Chen, Ya~Zhang, Yanfeng Wang, and Qi~Tian.
\newblock Symbiotic graph neural networks for 3d skeleton-based human action recognition and motion prediction.
\newblock {\em IEEE transactions on pattern analysis and machine intelligence}, 44(6):3316--3333, 2021.

\bibitem[LCZ{\etalchar{+}}23]{lcz+23}
Zirui Liu, Shengyuan Chen, Kaixiong Zhou, Daochen Zha, Xiao Huang, and Xia Hu.
\newblock Rsc: Accelerate graph neural networks training via randomized sparse computations.
\newblock {\em ICML}, 2023.

\bibitem[LGL22]{lgl22}
Xingchao Liu, Chengyue Gong, and Qiang Liu.
\newblock Flow straight and fast: Learning to generate and transfer data with rectified flow.
\newblock {\em arXiv preprint arXiv:2209.03003}, 2022.

\bibitem[LH17]{lh17}
Ilya Loshchilov and Frank Hutter.
\newblock Decoupled weight decay regularization.
\newblock {\em arXiv preprint arXiv:1711.05101}, 2017.

\bibitem[LLS{\etalchar{+}}25]{lls+25}
Xiaoyu Li, Yingyu Liang, Zhenmei Shi, Zhao Song, Wei Wang, and Jiahao Zhang.
\newblock On the computational capability of graph neural networks: A circuit complexity bound perspective.
\newblock {\em arXiv preprint arXiv:2501.06444}, 2025.

\bibitem[LMC{\etalchar{+}}24]{lmc+24}
Yang Lin, Xinyu Ma, Xu~Chu, Yujie Jin, Zhibang Yang, Yasha Wang, and Hong Mei.
\newblock Lora dropout as a sparsity regularizer for overfitting control.
\newblock {\em arXiv preprint arXiv:2404.09610}, 2024.

\bibitem[LMF20]{lmf20}
Zhaojiang Lin, Andrea Madotto, and Pascale Fung.
\newblock Exploring versatile generative language model via parameter-efficient transfer learning.
\newblock {\em arXiv preprint arXiv:2004.03829}, 2020.

\bibitem[LWH{\etalchar{+}}21]{lwh+21}
Zhenhua Liu, Yunhe Wang, Kai Han, Wei Zhang, Siwei Ma, and Wen Gao.
\newblock Post-training quantization for vision transformer.
\newblock {\em Advances in Neural Information Processing Systems}, 34:28092--28103, 2021.

\bibitem[LWS{\etalchar{+}}22]{lws+22}
Feng Liu, Han-Yang Wang, Si-Yuan Shen, Xun Jia, Jing-Yi Hu, Jia-Hao Zhang, Xi-Yi Wang, Ying Lei, Ai-Min Zhou, Jia-Yin Qi, et~al.
\newblock Opo-fcm: a computational affection based occ-pad-ocean federation cognitive modeling approach.
\newblock {\em IEEE Transactions on Computational Social Systems}, 10(4):1813--1825, 2022.

\bibitem[LWY{\etalchar{+}}24]{lwy+24}
Shih-Yang Liu, Chien-Yi Wang, Hongxu Yin, Pavlo Molchanov, Yu-Chiang~Frank Wang, Kwang-Ting Cheng, and Min-Hung Chen.
\newblock Dora: Weight-decomposed low-rank adaptation.
\newblock In {\em Forty-first International Conference on Machine Learning}, 2024.

\bibitem[LWZ{\etalchar{+}}22]{lwz+22}
Feng Liu, Hanyang Wang, Jiahao Zhang, Ziwang Fu, Aimin Zhou, Jiayin Qi, and Zhibin Li.
\newblock Evogan: An evolutionary computation assisted gan.
\newblock {\em Neurocomputing}, 469:81--90, 2022.

\bibitem[LZW{\etalchar{+}}24]{lzw+24}
Chengyi Liu, Jiahao Zhang, Shijie Wang, Wenqi Fan, and Qing Li.
\newblock Score-based generative diffusion models for social recommendations.
\newblock {\em arXiv preprint arXiv:2412.15579}, 2024.

\bibitem[MRG{\etalchar{+}}23]{mrg+23}
Chenlin Meng, Robin Rombach, Ruiqi Gao, Diederik Kingma, Stefano Ermon, Jonathan Ho, and Tim Salimans.
\newblock On distillation of guided diffusion models.
\newblock In {\em Proceedings of the IEEE/CVF Conference on Computer Vision and Pattern Recognition}, pages 14297--14306, 2023.

\bibitem[MRM20]{mrm20}
Christopher Morris, Gaurav Rattan, and Petra Mutzel.
\newblock Weisfeiler and leman go sparse: Towards scalable higher-order graph embeddings.
\newblock In {\em NeurIPS}, 2020.

\bibitem[MWZ24]{mwz24}
Fanxu Meng, Zhaohui Wang, and Muhan Zhang.
\newblock Pissa: Principal singular values and singular vectors adaptation of large language models.
\newblock {\em Advances in Neural Information Processing Systems}, 37:121038--121072, 2024.

\bibitem[NSJ{\etalchar{+}}22]{nsj+21}
S~Deepak Narayanan, Aditya Sinha, Prateek Jain, Purushottam Kar, and SUNDARARAJAN SELLAMANICKAM.
\newblock {IGLU}: Efficient {GCN} training via lazy updates.
\newblock In {\em International Conference on Learning Representations}, 2022.

\bibitem[ON15]{on15}
Keiron O'shea and Ryan Nash.
\newblock An introduction to convolutional neural networks.
\newblock {\em arXiv preprint arXiv:1511.08458}, 2015.

\bibitem[PGA{\etalchar{+}}24]{pga+24}
Bo~Peng, Daniel Goldstein, Quentin Anthony, Alon Albalak, Eric Alcaide, Stella Biderman, Eugene Cheah, Teddy Ferdinan, Haowen Hou, Przemys{\l}aw Kazienko, et~al.
\newblock Eagle and finch: Rwkv with matrix-valued states and dynamic recurrence.
\newblock {\em arXiv preprint arXiv:2404.05892}, 3, 2024.

\bibitem[PLD{\etalchar{+}}24]{pld+24}
Rui Pan, Xiang Liu, Shizhe Diao, Renjie Pi, Jipeng Zhang, Chi Han, and Tong Zhang.
\newblock Lisa: layerwise importance sampling for memory-efficient large language model fine-tuning.
\newblock {\em Advances in Neural Information Processing Systems}, 37:57018--57049, 2024.

\bibitem[RBL{\etalchar{+}}22]{rbl+22}
Robin Rombach, Andreas Blattmann, Dominik Lorenz, Patrick Esser, and Bj{\"o}rn Ommer.
\newblock High-resolution image synthesis with latent diffusion models.
\newblock In {\em Proceedings of the IEEE/CVF conference on computer vision and pattern recognition}, pages 10684--10695, 2022.

\bibitem[SATS24]{sats24}
Reece Shuttleworth, Jacob Andreas, Antonio Torralba, and Pratyusha Sharma.
\newblock Lora vs full fine-tuning: An illusion of equivalence.
\newblock {\em arXiv preprint arXiv:2410.21228}, 2024.

\bibitem[SHZ{\etalchar{+}}18]{shz+18}
Mark Sandler, Andrew Howard, Menglong Zhu, Andrey Zhmoginov, and Liang-Chieh Chen.
\newblock Mobilenetv2: Inverted residuals and linear bottlenecks.
\newblock In {\em Proceedings of the IEEE conference on computer vision and pattern recognition}, pages 4510--4520, 2018.

\bibitem[SSK18]{ssk18}
Abhishek Sinha, Mayank Singh, and Balaji Krishnamurthy.
\newblock Neural networks in an adversarial setting and ill-conditioned weight space.
\newblock In {\em Joint European Conference on Machine Learning and Knowledge Discovery in Databases}, pages 177--190. Springer, 2018.

\bibitem[SWL24]{swl24}
Hemanth Saratchandran, Thomas~X Wang, and Simon Lucey.
\newblock Weight conditioning for smooth optimization of neural networks.
\newblock In {\em European Conference on Computer Vision}, pages 310--325. Springer, 2024.

\bibitem[TLI{\etalchar{+}}23]{tli+23}
Hugo Touvron, Thibaut Lavril, Gautier Izacard, Xavier Martinet, Marie-Anne Lachaux, Timoth{\'e}e Lacroix, Baptiste Rozi{\`e}re, Naman Goyal, Eric Hambro, Faisal Azhar, et~al.
\newblock Llama: Open and efficient foundation language models.
\newblock {\em arXiv preprint arXiv:2302.13971}, 2023.

\bibitem[VCC{\etalchar{+}}18]{vcs+18}
Petar Veličković, Guillem Cucurull, Arantxa Casanova, Adriana Romero, Pietro Liò, and Yoshua Bengio.
\newblock Graph attention networks.
\newblock In {\em International Conference on Learning Representations}, 2018.

\bibitem[vdVSK24]{vsk24}
Gido~M van~de Ven, Nicholas Soures, and Dhireesha Kudithipudi.
\newblock Continual learning and catastrophic forgetting.
\newblock {\em arXiv preprint arXiv:2403.05175}, 2024.

\bibitem[WAC{\etalchar{+}}22]{wac+22}
Lijing Wang, Aniruddha Adiga, Jiangzhuo Chen, Adam Sadilek, Srinivasan Venkatramanan, and Madhav Marathe.
\newblock Causalgnn: Causal-based graph neural networks for spatio-temporal epidemic forecasting.
\newblock In {\em AAAI}, 2022.

\bibitem[WLW{\etalchar{+}}16]{wlw+16}
Jiaxiang Wu, Cong Leng, Yuhang Wang, Qinghao Hu, and Jian Cheng.
\newblock Quantized convolutional neural networks for mobile devices.
\newblock In {\em Proceedings of the IEEE conference on computer vision and pattern recognition}, pages 4820--4828, 2016.

\bibitem[WXHL24]{wxhl24}
Yibo Wen, Chenwei Xu, Jerry Yao-Chieh Hu, and Han Liu.
\newblock Alignab: Pareto-optimal energy alignment for designing nature-like antibodies.
\newblock {\em arXiv preprint arXiv:2412.20984}, 2024.

\bibitem[XHLJ19]{xhlj19}
Keyulu Xu, Weihua Hu, Jure Leskovec, and Stefanie Jegelka.
\newblock How powerful are graph neural networks?
\newblock In {\em International Conference on Learning Representations}, 2019.

\bibitem[XHT{\etalchar{+}}24]{xht+24}
Rui Xue, Haoyu Han, Mohamadali Torkamani, Jian Pei, and Xiaorui Liu.
\newblock Lazygnn: Large-scale graph neural networks via lazy propagation.
\newblock In {\em ICML}, 2024.

\bibitem[XLC{\etalchar{+}}24]{xlc+24}
Shuchen Xue, Zhaoqiang Liu, Fei Chen, Shifeng Zhang, Tianyang Hu, Enze Xie, and Zhenguo Li.
\newblock Accelerating diffusion sampling with optimized time steps.
\newblock In {\em Proceedings of the IEEE/CVF Conference on Computer Vision and Pattern Recognition}, pages 8292--8301, 2024.

\bibitem[XLZ{\etalchar{+}}21]{xlz+21}
Runxin Xu, Fuli Luo, Zhiyuan Zhang, Chuanqi Tan, Baobao Chang, Songfang Huang, and Fei Huang.
\newblock Raise a child in large language model: Towards effective and generalizable fine-tuning.
\newblock {\em arXiv preprint arXiv:2109.05687}, 2021.

\bibitem[XWL{\etalchar{+}}22]{xwl+22}
Weizhi Xu, Junfei Wu, Qiang Liu, Shu Wu, and Liang Wang.
\newblock Evidence-aware fake news detection with graph neural networks.
\newblock In {\em Proceedings of the ACM web conference 2022}, pages 2501--2510, 2022.

\bibitem[XZZ{\etalchar{+}}21]{xzz+21}
Zhiyong Xu, Weicun Zhang, Tianxiang Zhang, Zhifang Yang, and Jiangyun Li.
\newblock Efficient transformer for remote sensing image segmentation.
\newblock {\em Remote Sensing}, 13(18):3585, 2021.

\bibitem[YHW{\etalchar{+}}22]{yhw+22}
Fang Yu, Kun Huang, Meng Wang, Yuan Cheng, Wei Chu, and Li~Cui.
\newblock Width \& depth pruning for vision transformers.
\newblock In {\em Proceedings of the AAAI Conference on Artificial Intelligence}, volume~36, pages 3143--3151, 2022.

\bibitem[YJS{\etalchar{+}}23]{yjs+23}
Longhui Yu, Weisen Jiang, Han Shi, Jincheng Yu, Zhengying Liu, Yu~Zhang, James~T Kwok, Zhenguo Li, Adrian Weller, and Weiyang Liu.
\newblock Metamath: Bootstrap your own mathematical questions for large language models.
\newblock {\em arXiv preprint arXiv:2309.12284}, 2023.

\bibitem[YW25]{yw25}
Lu~Yi and Zhewei Wei.
\newblock Scalable and certifiable graph unlearning: Overcoming the approximation error barrier.
\newblock In {\em The Thirteenth International Conference on Learning Representations}, 2025.

\bibitem[ZCB{\etalchar{+}}23]{zcb+23}
Qingru Zhang, Minshuo Chen, Alexander Bukharin, Nikos Karampatziakis, Pengcheng He, Yu~Cheng, Weizhu Chen, and Tuo Zhao.
\newblock Adalora: Adaptive budget allocation for parameter-efficient fine-tuning.
\newblock {\em arXiv preprint arXiv:2303.10512}, 2023.

\bibitem[Zha24]{zha24}
Jiahao Zhang.
\newblock Graph unlearning with efficient partial retraining.
\newblock In {\em Companion Proceedings of the ACM on Web Conference 2024}, pages 1218--1221, 2024.

\bibitem[ZLSS22]{zlss22}
Shichang Zhang, Yozen Liu, Yizhou Sun, and Neil Shah.
\newblock Graph-less neural networks: Teaching old {MLP}s new tricks via distillation.
\newblock In {\em International Conference on Learning Representations}, 2022.

\bibitem[ZLZ21a]{zlz21_micro}
Jiahao Zhang, Feng Liu, and Aimin Zhou.
\newblock Off-tanet: A lightweight neural micro-expression recognizer with optical flow features and integrated attention mechanism.
\newblock In {\em Pacific Rim International Conference on Artificial Intelligence}, pages 266--279. Springer, 2021.

\bibitem[ZLZ21b]{zlz21_face}
Zengqun Zhao, Qingshan Liu, and Feng Zhou.
\newblock Robust lightweight facial expression recognition network with label distribution training.
\newblock In {\em AAAI}, 2021.

\bibitem[ZWZ{\etalchar{+}}22]{zwz+22}
Zeyang Zhang, Xin Wang, Ziwei Zhang, Haoyang Li, Zhou Qin, and Wenwu Zhu.
\newblock Dynamic graph neural networks under spatio-temporal distribution shift.
\newblock In {\em NeurIPS}, 2022.

\bibitem[ZXF{\etalchar{+}}24]{zxf+24}
Jiahao Zhang, Rui Xue, Wenqi Fan, Xin Xu, Qing Li, Jian Pei, and Xiaorui Liu.
\newblock Linear-time graph neural networks for scalable recommendations.
\newblock In {\em Proceedings of the ACM on Web Conference 2024}, pages 3533--3544, 2024.

\bibitem[ZZC{\etalchar{+}}24]{zzc+24}
Jiawei Zhao, Zhenyu Zhang, Beidi Chen, Zhangyang Wang, Anima Anandkumar, and Yuandong Tian.
\newblock Galore: Memory-efficient llm training by gradient low-rank projection.
\newblock {\em arXiv preprint arXiv:2403.03507}, 2024.

\bibitem[ZZS{\etalchar{+}}24]{zzs+24}
Tianyu Zheng, Ge~Zhang, Tianhao Shen, Xueling Liu, Bill~Yuchen Lin, Jie Fu, Wenhu Chen, and Xiang Yue.
\newblock Opencodeinterpreter: Integrating code generation with execution and refinement.
\newblock {\em arXiv preprint arXiv:2402.14658}, 2024.

\end{thebibliography}
